\documentclass[10pt,twocolumn,letterpaper]{article}

\usepackage{cvpr}
\usepackage[utf8]{inputenc}
\usepackage{graphicx}
\usepackage{color}
\usepackage{times}
\usepackage{amsmath,mathtools}
\usepackage{amsthm}	
\usepackage{amssymb}
\usepackage{enumerate}
\usepackage{subfigure}
\usepackage[algo2e,inoutnumbered,linesnumbered,algoruled,vlined]{algorithm2e}
\usepackage{algpseudocode}
\usepackage{booktabs}
\usepackage{tabularx}
\usepackage{array}
\usepackage{bm}
\usepackage{url}
\usepackage{bbm}
\usepackage{enumitem}
\usepackage{xcolor}
\usepackage{stackengine}
\usepackage{wrapfig}
\usepackage{multirow}
\usepackage{multicol}
\usepackage{tablefootnote}

\stackMath

\usepackage[breaklinks=true,bookmarks=false]{hyperref}

\cvprfinalcopy 

\author{
  Tolga Birdal\textsuperscript{ 1}
  \qquad Umut \c{S}im\c{s}ekli\textsuperscript{ 2}
  \\
  \textsuperscript{1 }{Fakult\"at für Informatik, Technische Universit\"{a}t M\"{u}nchen, 85748 M\"{u}nchen, Germany}\\
  \textsuperscript{2 }{LTCI, T\'{e}l\'{e}com ParisTech, Universit\'{e} Paris-Saclay, 75013 Paris, France}
}

\usepackage{tabularx}
\newcommand{\PBS}[1]{\let\temp=\\#1\let\\=\temp}
\newcommand{\RBS}{\let\\=\tabularnewline}

\DeclareMathSymbol{@}{\mathord}{letters}{"3B}

\def\cvprPaperID{108} 

\newcommand{\x}{\mathbf{x}}
\newcommand{\thb}{\mathbf{x}} 

\newcommand{\xe}{\tilde{\mathbf{X}}}

\newcommand{\B}{\mathcal{B}}
\newcommand{\Id}{\mathbf{I}}
\newcommand{\Man}{\mathcal{M}}
\newcommand{\T}{\mathcal{T}}
\newcommand{\DP}{\mathcal{DP}}

\newcommand{\Pm}{\mathbf{P}}

\newcommand{\one}{\mathbf{1}}
\newcommand{\zero}{\mathbf{0}}
\newcommand{\Eye}{\mathbf{I}}
\newcommand{\X}{\mathbf{X}}
\newcommand{\Z}{\mathbf{Z}}
\newcommand{\Xxi}{\bm{\xi}_{\X}}

\newcommand{\sing}{\mathbf{S}}

\newcommand{\Om}{\mathbf{O}}

\newcommand{\Og}{\mathcal{O}}

\newcommand{\y}{\mathbf{y}}
\newcommand{\s}{\mathbf{s}}
\newcommand{\Y}{{\mathbf{Y}}}

\newcommand{\Scal}{\mathcal{S}}
\newcommand{\R}{\mathbb{R}}

\newcommand{\Edge}{\mathcal{E}}

\newtheorem{thm}{Theorem}

\newtheorem{prop}{Proposition}
\newtheorem{dfn}{Definition}

\DeclareMathOperator*{\argmin}{arg\min}
\DeclareMathOperator*{\argmax}{arg\max}

\usepackage{cleveref}

\Crefname{assumption}{\textbf{H}\hspace{-3pt}}{\textbf{H}\hspace{-3pt}}
\crefname{assumption}{\textbf{H}}{\textbf{H}}

\newcommand{\insertimageC}[5]{ 
\begin{figure}[#5]
\centering
\includegraphics[width=#1\linewidth, clip=true]{figures/#2}
\caption{#3}
\label{#4}
\end{figure}
}

\newcommand{\insertimageStar}[5]{ 
\begin{figure*}[#5]
\centering
\includegraphics[width=#1\linewidth, clip=true]{figures/#2}
\caption{#3}
\label{#4}
\end{figure*}
}

\ifcvprfinal\fi
\begin{document}

\title{Probabilistic Permutation Synchronization using the Riemannian Structure of the Birkhoff Polytope}

\maketitle

\begin{abstract}
We present an entirely new geometric and probabilistic approach to synchronization of correspondences across multiple sets of objects or images. In particular, we present two algorithms: (1) Birkhoff-Riemannian L-BFGS for optimizing the relaxed version of the combinatorially intractable cycle consistency loss in a principled manner, (2) Birkhoff-Riemannian Langevin Monte Carlo for generating samples on the Birkhoff Polytope and estimating the confidence of the found solutions. To this end, we first introduce the very recently developed Riemannian geometry of the Birkhoff Polytope. Next, we introduce a new probabilistic synchronization model in the form of a Markov Random Field (MRF). Finally, based on the first order retraction operators, we formulate our problem as simulating a stochastic differential equation and devise new integrators. We show on both synthetic and real datasets that we achieve high quality multi-graph matching results with faster convergence and reliable confidence/uncertainty estimates.\vspace{-2mm} 
\end{abstract}

\vspace{-3mm}\section{Introduction}
\label{sec:intro}
Correspondences fuel a large variety of computer vision applications such as structure-from-motion (SfM)~\cite{schonberger2016structure}, SLAM~\cite{murORB2}, 3D reconstruction~\cite{dai2017bundlefusion,birdal2017cad,birdal2016online}, camera re-localization~\cite{Sattler12imageretrieval}, image retrieval~\cite{li2015pairwise} and 3D scan stitching~\cite{huber2003fully,deng2018ppf}. In a typical scenario, given two scenes, an initial set of 2D/3D keypoints is first identified. Then the neighborhood of each keypoint is summarized with a local descriptor ~\cite{lowe2004distinctive,Deng_2018_CVPR} and keypoints in the given scenes are matched by associating the mutually closest descriptors. 
In a majority of practical applications, multiple images or 3D shapes are under consideration and ascertaining such two-view or \textit{pairwise} correspondences is simply not sufficient. This necessitates a further refinement ensuring global consistency.
Unfortunately, at this stage even the well developed pipelines acquiesce either heuristic/greedy refinement~\cite{degol2018improved} or incorporate costly geometric cues related to the linking of individual correspondence estimates into a globally coherent whole~\cite{opensfm,schonberger2016structure,triggs1999bundle}. 
\insertimageC{1}{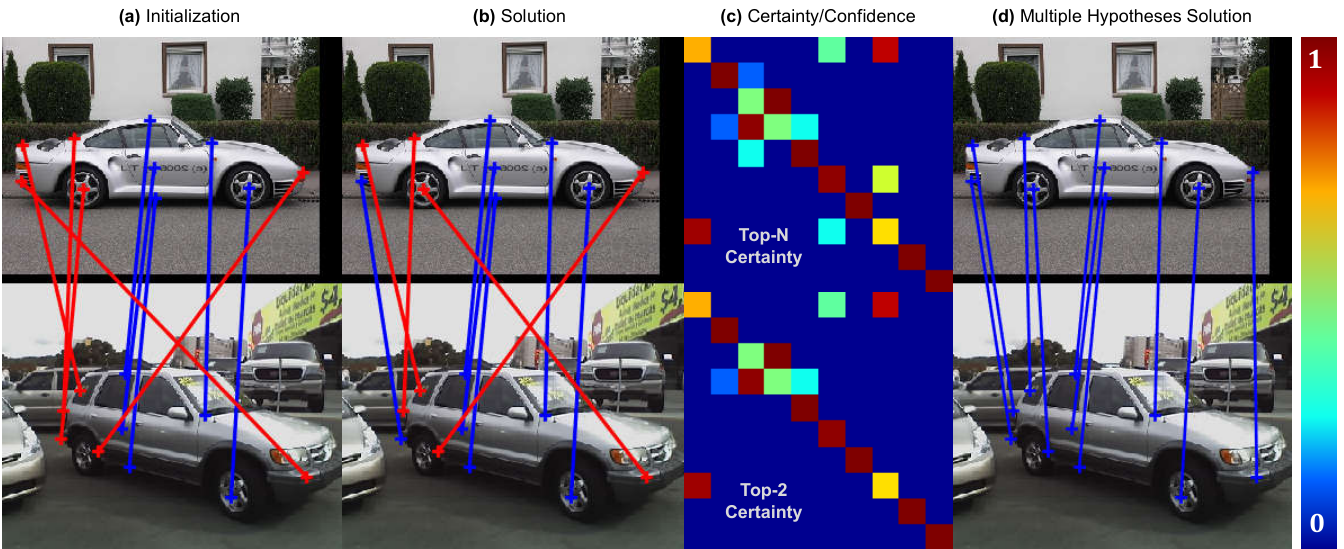}{Our algorithm robustly solves the multiway image matching problem (\textbf{a}, \textbf{b}) and provides confidence maps (\textbf{c}) that can be of great help in further improving the estimates (\textbf{d}). The bar on the right is used to assign colors to confidences. For the rest, incorrect matches are marked in red and correct ones in blue.\vspace{-4mm}}{fig:teaser}{t!}

In this paper, by using the fact that correspondences are \textit{cycle consistent}~\footnote{Composition of correspondences for any circular path arrives back at the start node.}, we propose two novel algorithms for refining the assignments across multiple images/scans (nodes) in a \textit{multi-way graph} and for estimating assignment confidences, respectively. 
We model the correspondences between image pairs as \textit{relative}, \textit{total} permutation matrices and seek to find \textit{absolute} permutations that re-arrange the detected keypoints to a single canonical, global order. This problem is known as \textit{map} or \textit{permutation synchronization}~\cite{pachauri2013solving,sun2018joint}. Even though in many practical scenarios matches are only partially available, when shapes are complete and the density of matches increases, total permutations can suffice~\cite{huang2013consistent}. 

Similar to many well received works~\cite{zaslavskiy2009path,schiavinato2017synchronization}, we relax the sought permutations to the set of doubly-stochastic (DS) matrices. We then consider the geometric structure of DS, the \textit{Birkhoff Polytope}~\cite{birkhoff1946tres}. We are - to the best of our knowledge, for the first time introducing and applying the recently developed Riemannian geometry of the Birkhoff Polytope~\cite{douik2018} to tackle challenging problems of computer vision. Note that lack of this geometric understanding caused plenty of obstacles for scholars dealing with our problem~\cite{schiavinato2017synchronization,wang2018multi}. By the virtue of a first order retraction, we can use the recent Riemannian limited-memory BFGS (LR-BFGS) algorithm~\cite{yuan2016riemannian} to perform a maximum-a-posteriori (MAP) estimation of the parameters of the consistency loss. 
We coin our variation as \textit{Birkhoff-LRBFGS}.

At the next stage, we take on the challenge of confidence/uncertainty estimation for the problem at hand by drawing samples on the Birkhoff Polytope and estimating the empirical posterior distribution. To achieve this, we first formulate a new geodesic stochastic differential equation (SDE). Our SDE is based upon the Riemannian Langevin Monte Carlo (RLMC)~\cite{girolami2011riemann,xifara2014langevin,patterson2013stochastic} that is efficient and effective in sampling from Riemannian manifolds with \textit{true} exponential maps. Note that similar stochastic gradient geodesic MCMC (SG-MCMC)~\cite{Liu2016,byrne2013} tools have already been used in the context of synchronization of spatial rigid transformations whose parameters admit an analytically defined geodesic flow~\cite{birdalSimsekli2018}. Unfortunately, for our manifold the retraction map is only up to first order and hence we cannot use off-the-shelf schemes. Alleviating this nuisance, we further contribute a novel numerical integrator to solve our SDE by replacing the intractable exponential map of DS matrices by the approximate retraction map. This leads to another new algorithm: \textit{Birkhoff-RLMC}.

In a nutshell, our contributions are:
\setlist{nolistsep}
\begin{enumerate}
\itemsep0em 
    \item We function as an ambassador and introduce the Riemannian geometry of the Birkhoff Polytope~\cite{douik2018} to solve problems in computer vision.
    \item We propose a new probabilistic model for the permutation synchronization problem.
    \item We minimize the cycle consistency loss via a Riemannian-LBFGS algorithm and outperfom the state-of-the-art both in recall and in runtime.
    \item Based upon the Langevin mechanics, we introduce a new SDE and a numerical integrator to draw samples on the high dimensional and complex manifolds with approximate retractions, such as the Birkhoff Polytope. This lets us estimate the confidence maps, which can aid in improving the solutions and spotting consistency violations or outliers.
\end{enumerate}

Note that the tools developed herewith can easily extend beyond our application and would hopefully facilitate promising research directions regarding the combinatorial optimization problems in computer vision.

\section{Related Work}
\label{sec:related}
Permutation synchronization is an emerging domain of study due to its wide applicability, especially for the problems in computer vision. We now review the developments in this field, as chronologically as possible. Note that multiway graph matching problem formulations involving spatial geometry are well studied~\cite{cosmo2016game,maciel2003global,li20073d,fathony2018efficient,yan2016multi,cosmo2017consistent}, as well as transformation synchronization~\cite{wang2013exact,chaudhury2015global,thunberg2017distributed,tron2014distributed,arrigoni2016spectral,bernard2015solution,govindu2004lie}.
For brevity, we omit these literature and focus on works that explicitly operate on correspondence matrices.

The first applications of \textit{synchronization}, a term coined by Singer~\cite{singer2011three,singer2011angular}, to correspondences only date back to early 2010s~\cite{nguyen2011optimization}. Pachauri~\etal~\cite{pachauri2013solving} gave a formal definition and devised a spectral technique. The same authors quickly extended their work to Permutation Diffusion Maps~\cite{pachauri2014permutation} finding correspondence between images. Unfortunately, this first method was quadratic in the number of images and hence was not computationally friendly. 
In a sequel of works called \textit{MatchLift}, Huang, Chen and Guibas~\cite{huang2013consistent,Chen2014near} were the firsts to cast the problem of estimating cycle-consistent maps as finding the closest positive semidefinite matrix to an input matrix. They also addressed the case of partial permutations. Due to the semidefinite programming (SDP) involved, this perspective suffered from high computational cost in real applications. Similar to Pachauri~\cite{pachauri2013solving}, for $N$ images and $M$ edges, this method required computing an eigendecomposition of an $NM\times NM$ matrix.
Zhou~\etal~\cite{zhou2015multi} then introduced \textit{MatchALS}, a new low-rank formulation with nuclear-norm relaxation, globally solving the joint matching of a set of images without the need of SDP. Yu~\etal~\cite{yu2016globally} formulated a synchronization energy similar to our method and proposed proximal Gauss-Seidel methods for solving a relaxed problem. 
However, unlike us, this paper did not use the geometry of the constraints or variables and thereby resorted to complicated optimization procedures involving Frank-Wolfe subproblems for global constraint satisfaction.
Arrigoni~\etal~\cite{arrigoni2017synchronization} and Maset~\etal~\cite{arrigoni2017synchronization} extended Pachauri~\cite{pachauri2013solving} to operate on partial permutations using spectral decomposition. To do so, they considered the symmetric inverse semigroup of the partial matches that are typically hard to handle. Their closed form methods did not need initialization steps to synchronize, but also did not establish an explicit cycle consistency.
Tang~\etal~\cite{tang2017initialization} opted to use ordering heuristics improving upon Pachauri~\cite{pachauri2013solving}. 
Cosmo~\etal~\cite{cosmo2017consistent} brought an interesting solution to the problem of estimating consistent correspondences between multiple 3D shapes, without requiring initial pairwise solutions as input. 
Schiavinato and Torsello~\cite{schiavinato2017synchronization} tried to overcome the lack of group structure of the Birkhoff polytope by transforming any graph-matching problem into a multi-graph matching one.
Bernard~\etal~\cite{bernard2018} used an NMF-based approach to generate a cycle-consistent synchronization.
Park and Yoon~\cite{PARK2018} used multi-layer random walks framework to address the global correspondence search problem of multi-attributed graphs. Starting from a multi-layer random-walks initialization, the authors proposed a robust solver by iterative reweighting. 
Hu~\etal~\cite{hu2018distributable} revisited the \textit{MatchLift} and developed a scalable, distributed solution with the help of ADMMs, called \textit{DMatch}. Their idea of splitting the input into sub-collections can still lead to global consistency under mild conditions while improving the efficiency. Finally, Wang~\etal~\cite{wang2018multi} made use of the domain knowledge and added the geometric consistency of image coordinates as a low-rank term to increase the recall.

The aforementioned works have neither considered the Riemennian structure of the common Birkhoff convex relaxation nor have they provided a probabilistic framework, which can pave the way to uncertainty estimation while simultaneously solving the optimization problem. This is what we propose in this work.

\section{Preliminaries and Technical Background}



\begin{dfn}[Permutation Matrix]
A \textit{permutation matrix} is defined as a sparse, square binary matrix, where each column and each row contains only a single \textit{true (1)} value:
\begin{equation}
\mathcal{P}_n := \{\Pm \in \{0,1\}^{n\times n} : \Pm \one_n = \one_n\,,\,\one_n^\top \Pm = \one_n^\top\}.
\end{equation}
where $\one_n$ denotes a $n$-dimensional ones vector. Every $\Pm \in \mathcal{P}^n$ is a \textit{total} permutation matrix and $P_{ij}=1$ implies that element $i$ is mapped to element $j$. Permutation matrices are the only strictly non-negative elements of the orthogonal group $\mathcal{P}_n\in \mathcal{O}_n = \{\Om : \Om^\top\Om=\Eye\}$, a special case of the Stiefel manifold of $m$—frames in $\mathbb{R}_n$ when $m=n$.
\end{dfn}
\begin{dfn}[Center of Mass]
\label{dfn:center}
The center of mass for all the permutations on $n$ objects is defined in $\R^{n\times n}$ as~\cite{plis2011directional}:
\begin{equation}
\mathbf{C}_n = \frac{1}{n!} \sum\nolimits_{\Pm_i \in \mathcal{P}_n} \Pm_i=\frac{1}{n!}(n-1)! \one_n \one_n^\top = \frac{1}{n}\one_n \one_n^\top.
\end{equation}
Notice that $\mathbf{C}_n \notin \mathcal{P}^n$ as shown in Fig.~\ref{fig:illustrations}.
\end{dfn}
\begin{dfn}[Relative Permutation]
We define a permutation matrix to be \textbf{relative} if it is the ratio (or difference) of two group elements $(i\rightarrow j)$: $\Pm_{ij}=\Pm_i\Pm_j^\top$.
\end{dfn}
\begin{dfn}[Permutation Synchronization Problem]
Given a redundant set of measures of ratios $\{\Pm_{ij}\}\,:\, (i,j)\in \Edge \subset \{1,\dots, N\} \times \{1,\dots, N\} $, where $\Edge$ denotes the set of the edges of a directed graph of $N$ nodes, the permutation synchronization~\cite{pachauri2013solving} can be formulated as the problem of recovering $\{\Pm_i\}$ for $i=1,\dots,N$ such that the group consistency constraint is satisfied: $\Pm_{ij}=\Pm_i\Pm_j^{-1}$.
\end{dfn}
If the input data is noise-corrupted, this \textit{consistency} will not hold and to recover the \textit{absolute permutations} $\{\Pm_i\}$, some form of a \textit{consistency error} is minimized. 
Typically, any form of minimization on the discrete space of permutations is intractable and these matrices are relaxed by their \textit{doubly-stochastic} counterparts~\cite{busam2015acvr,zaslavskiy2009path,schiavinato2017synchronization} (see Fig.~\ref{fig:illustrations}). 
\begin{dfn}[Doubly Stochastic (DS) Matrix]
A DS matrix is a non-negative, square matrix whose rows and columns sum to $1$. The set of DS matrices is defined as:
\begin{align}
\DP_n = \{\,\,\X \in \mathbb{R}_+^{n \times n} : \sum\limits_{i=1}^n x_{ij}=1 \,\wedge\, \sum\limits_{j=1}^n x_{ij}=1 \,\,\}.
\end{align}
\end{dfn}
\begin{thm}[Birkhoff-von Neumann Theorem]
\label{thm:bvn}
The convex hull of the set of all permutation matrices is the set of doubly-stochastic matrices and there exists a potentially non-unique $\bm{\theta}$ such that any DS matrix can be expressed as a linear combination of $k$ permutation matrices~\cite{birkhoff1946tres,hurlbert2008short}:
\begin{align}
\label{eq:BVN}
    \X = \theta_1 \Pm_1+\dots+\theta_k \Pm_k\,,\, \theta_i>0 \,\wedge\, \bm{\theta}^\top\one_k = 1.
\end{align}
While finding the minimum $k$ is shown to be NP-hard~\cite{dufosse2017further}, by \textit{Marcus-Ree theorem}, we know that there exists one constructible decomposition where $k<(n-1)^2+1$. 
\end{thm}
\begin{dfn}[Birkhoff Polytope]
The multinomial manifold of DS matrices is incident to the convex object called the Birkhoff Polytope~\cite{birkhoff1946tres}, an $(n - 1)^2$ dimensional convex submanifold of the ambient $\mathbb{R}^{n\times n}$ with $n!$ vertices: $\mathcal{B}_n\equiv\DP_n$. We use $\DP_n$ to refer to the Birkhoff Polytope.
\end{dfn}

It is interesting to see that this convex polytope is co-centered with $\mathcal{P}_n$ at $\mathbf{C}_n$, $\mathbf{C}_n\in\DP_n$ and over-parameterizes the convex hull of the permutation vectors, the \textit{permutahedron}~\cite{goemans2015smallest}. 
$\mathcal{P}_n$ can now be considered as an orthogonal subset of $\DP_n$: $\mathcal{P}_n=\{\X \in \DP_n : \X \X^\top=\mathbf{I}\}$, i.e. the discrete set of permutation matrices is the intersection of the convex set of DS matrices and the $\mathcal{O}_n$.
\insertimageC{1}{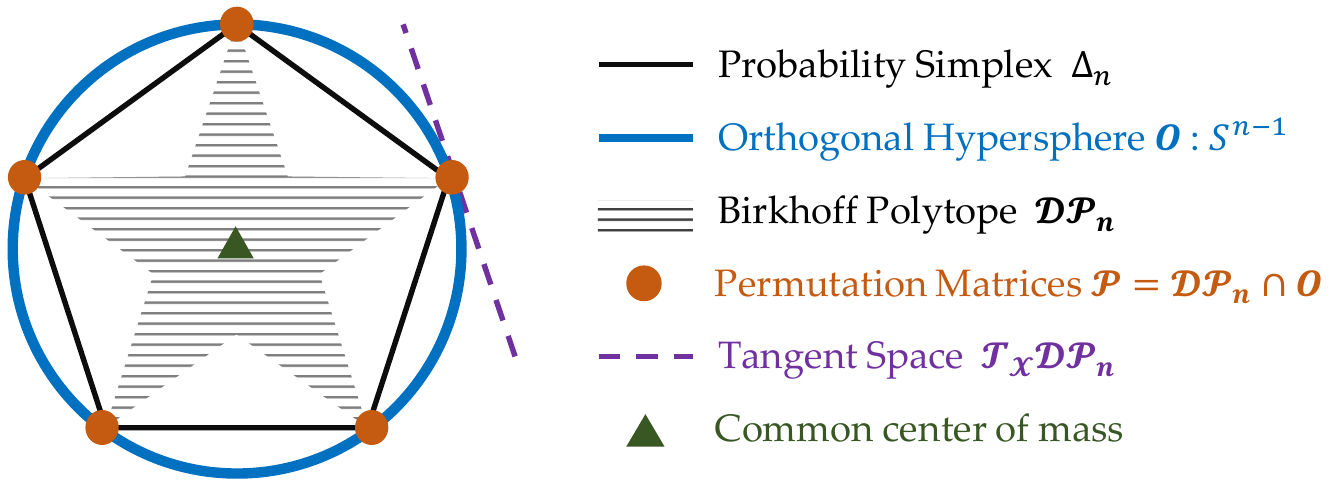}{Simplified (matrices are vectorized) illustration of geometries we consider: (i) $\Delta_n$ is convex, (ii) $\mathcal{DP}_n$ is strictly contained in $\Delta_n$. 
In low dimensions, such configuration cannot exist as there is no convex shape that touches $\Delta_n$ only on the corners.\vspace{-4mm}}{fig:illustrations}{t!}
\subsection{Riemannian Geometry of the Birkhoff Polytope}
Recently, Douik~\etal~\cite{douik2018} endowed $\DP_n$ with the Fisher information metric, resulting in the Riemannian manifold of $\DP_n$. To the best of our knowledge, we are the first to exploit this manifold in the domain of computer vision, and hence will now recall the main results of Douik~\etal~\cite{douik2018} and summarize the main constructs of Riemannian optimization on $\DP_n$. The proofs can be found in~\cite{douik2018}.
\begin{dfn}[Tangent Space and Bundle]
The tangent bundle is referred to as the union of all tangent spaces $
\mathcal{T}\DP_n=\cup_{\X \in \DP_n}\mathcal{T}_{\X}\DP_n$ one of which is defined as:
\begin{equation}
\mathcal{T}_{\X} \DP_n := \{\Z \in \mathbb{R}^{n\times n} : \Z \one_n = \zero_n\,,\,\Z^\top\one_n = \zero_n\}.
\end{equation}
\end{dfn}
\begin{thm}
The projection operator $\Pi_\X(\Y),\Y \in \DP_n$ onto the tangent space of $\X \in \DP_n$, $\mathcal{T}_{\X} \DP_n$ is written as:\vspace{-8mm}
\end{thm}
\begin{equation}\vspace{-5mm}
\Pi_\X(\Y) = \Y - (\bm{\alpha} \one^\top + \one \bm{\beta}^\top ) \odot \X, \quad \text{with}
\end{equation}\vspace{-6mm}
\begin{align*}
\bm{\alpha} = (\mathbf{I}-\X\X^\top)^+(\Y-\X\Y^\top)\one,\quad\bm{\beta} =\Y^\top\one-\X^\top\bm{\alpha},
\end{align*}
\noindent$^+$ depicts the left pseudo-inverse and $\odot$ the Hadamard product. Note that there exists a numerically more stable way to compute the same concise formulation of $\Pi_\X(\Y)$~\cite{douik2018}.
\begin{thm}
\label{thm:retract}
For a vector $\Xxi \in \mathcal{T}_\X \DP_n$ lying on the tangent space of $\X \in \DP_n$, the first order retraction map $R_\X$ is given as follows:
\begin{align}
    R_\X(\Xxi) = \Pi(\X\odot \exp(\Xxi \oslash \X)),
\end{align}
where the operator $\Pi$ denotes the projection onto $\DP_n$, efficiently computed using the Sinkhorn algorithm~\cite{sinkhorn1967concerning} and $\oslash$ is the Hadamard division.
\end{thm}

Plis~\etal~\cite{plis2011directional} showed that on the $n$-dimensional Birkhoff Polytope all permutations are equidistant from the center of mass $\mathbf{C}_n$, and thus the extreme points of $\DP_n$, that are the permutation matrices, are located on an $(n-1)^2$-dimensional hypersphere $S^{(n-1)^2}$ of radius $\sqrt{n-1}$, centered at $\mathbf{C}_n$. This hypersphere is incident to the Birkhoff Polytope on the vertices.

\begin{prop}
The gap as a ratio between $\DP_n$ and both $\mathcal{S}^{(n-1)^2}$ and $\mathcal{O}_n$ grows to infinity as $n$ grows.
\label{thm:gap}
\end{prop}
\noindent The proof is given in the supplementary document. While there exists polynomial time projections of the $n!$-element permutation space onto the continuous hypersphere representation and back~\cite{plis2011directional}, Prop.~\ref{thm:gap} prevents us from using hypersphere relaxations, as done in preceding works~\cite{plis2011directional,Zanfir_2018_CVPR}.


\section{Proposed Probabilistic Model}
We assume that we are provided a set of pairwise, \textit{total} permutations $\Pm_{ij} \in \mathcal{P}_n$ for $(i,j) \in \Edge$ and we are interested in finding the underlying \emph{absolute permutations} $\X_i$ for $i\in\{1,\dots,N\}$ with respect to a common origin (e.g.\ $\X_1=\Id$, the identity matrix). We seek absolute permutations that would respect the consistency of the underlying graph structure. For conciseness, we also restrict our setting to total permutations, and leave the extension to partial permutations, which live on the \textit{monoid}, as a future study. Because operating directly on $\mathcal{P}_n$ would require us to solve a combinatorial optimization problem and because of the lack of a manifold structure for $\mathcal{P}_n$, we follow the popular approach~\cite{linderman2018reparameterizing,yan2016short,lyzinski2016graph} and relax the domain of the absolute permutations by assuming that each $\X_i \in \DP_n$. 

We formulate the permutation synchronization problem in a probabilistic context where we treat the pairwise relative permutations as \emph{observed} random variables and the absolute ones as \textit{latent} random variables. In particular, our probabilistic construction enables us to cast the synchronization problem as inferential in the model.
With a slight abuse of notation, in the rest of the paper, we will denote $\Pm \equiv \{\Pm_{ij}\}_{(i,j)\in \Edge}$ and $\X \equiv \{\X_i\}_{i=1}^N$, all the observations and all the latent variables, respectively. 

A typical way to build a probabilistic model is to first choose the prior distributions on $\mathcal{DP}_n$ for each $\X_{i}$ and then choose a conditional distribution on $\mathcal{P}_n$ for each $\X_{ij}$ given the latent variables. Unfortunately, standard parametric distributions neither exist on $\mathcal{DP}_n$ nor on $\mathcal{P}_n$. The variational \textit{stick breaking}~\cite{linderman2018reparameterizing} yields an \emph{implicitly} defined PDF on $\DP_n$ and is not able to provide direct control on the resulting distribution. Defining Kantorovich distance-based distributions over the permutation matrices is possible~\cite{clemenccon2010kantorovich}, yet these models incur high computational costs since they would require solving optimal transport problems during inference. For these reasons, instead of constructing a hierarchical probabilistic model, we will directly model the full joint distribution of $\Pm$ and $\X$.



We propose a probabilistic model where we assume the full joint distribution admits the following factorized form:
\begin{align}
    p(\Pm, \X) = \frac{1}{Z} \prod\nolimits_{(i,j) \in \Edge} \psi(\Pm_{ij}, \X_i, \X_j), \label{eqn:probmodel}
\end{align}
where $Z$ denotes the normalization constant with
\begin{align}
Z := \sum\limits_{\Pm \in \mathcal{P}_n^{|\Edge|}} \int_{\DP_n^N} \prod\limits_{(i,j) \in \Edge} \psi(\Pm_{ij}, \X_i, \X_j) \> \mathrm{d} \X,
\end{align} 
and $\psi$ is called the `clique potential' that is defined as:
\begin{align}
\psi(\Pm_{ij}, \X_i, \X_j) \triangleq \exp(-\beta \| \Pm_{ij} - \X_i \X_j^\top \|^2_\mathrm{F}).
\end{align}
Here $\|\cdot\|_\mathrm{F}$ denotes the Frobenius norm, $\beta \in \mathbb{R}_+$ is the \emph{dispersion} parameter that controls the spread of the distribution. 
Note that the model is a Markov random field \cite{kindermann1980markov}. 

Let us take a closer look at the proposed model. If we define $\X_{ij} := \X_i \X_j^\top \in \DP_n$, then by Thm.~\ref{thm:bvn}, we have the following decomposition for each $\X_{ij}$:
\begin{align}
\X_{ij} = \sum\nolimits_{b=1}^{B_{ij}} \theta_{ij,b} \mathbf{M}_{ij,b}, \quad \sum\nolimits_{b=1}^{B_{ij}} \theta_{ij,b} =1,
\end{align}
where $B_{ij}$ is a positive integer, each $\theta_{ij,b} \geq 0$, and $\mathbf{M}_{ij,b} \in \mathcal{P}_n$. The next result states that we have an equivalent hierarchical interpretation for the proposed model:
\begin{prop}
The probabilistic model defined in Eq.~\ref{eqn:probmodel} implies the following hierarchical decomposition:
\begin{align}
&p(\X) = \frac1{C} \exp \Bigl(-\beta \hspace{-5pt} \sum\limits_{(i,j)\in \Edge}\|\X_{ij}\|^2 \Bigr) \prod_{(i,j)\in \Edge} Z_{ij} \\
&p(\Pm_{ij} | \X_i, \X_j) = \frac1{Z_{ij}} \exp\Bigl( 2\beta\>\mathrm{tr}(\Pm_{ij}^\top \X_{ij}) \Bigr)  
\end{align}
where $C$ and $Z_{ij}$ are normalization constants. Besides, for all $i,j$, $Z_{ij} \geq \prod_{b=1}^{B_{ij}}  f(\beta, \theta_{ij,b}) $, where $f$ is a positive function that is increasing in both $\beta$ and $\theta_{ij,b}$.
\end{prop}
%
The proof is given in the supplementary and is based on the simple decomposition $p(\Pm,\X) = p(\X)p(\Pm|\X)$. This hierarchical point of view lets us observe some interesting properties: (\textbf{1}) the likelihood $p(\Pm_{ij} | \X_i, \X_j)$ mainly depends on the term $\mathrm{tr}(\Pm_{ij}^\top \X_{ij})$ that measures the data fitness. We aptly call this term the `soft Hamming distance' between $\Pm_{ij}$ and $\X_{ij}$ since it would correspond to the actual Hamming distance between two permutations if $\X_i, \X_j$ were permutation matrices \cite{korba2018structured}. (\textbf{2}) On the other hand, the prior distribution contains two competing terms: (i) the term $Z_{ij}$ favors large $\theta_{ij,b}$, which would push $\X_{ij}$ towards the corners of the Birkhoff polytope, (ii) the term $\|\X_{ij}\|^2_\mathrm{F}$ acts as a regularizer on the latent variables and attracts them towards the center of the Birkhoff polytope $\mathbf{C}_n$ (cf.\ Dfn.~\ref{dfn:center}), which will be numerically beneficial for the inference algorithms that will be developed in the following section.



\section{Inference Algorithms}

We can now formulate the permutation synchronization problem as a probabilistic inference problem, where we will be interested in the following quantities: 
\begin{enumerate}[itemsep=4pt,topsep=2pt,leftmargin=*]
\item Maximum a-posteriori (MAP): 
\begin{align}
\X^\star = \argmax\limits_{\X \in \DP_n^N} \log p(\X | \Pm )
\end{align}
where 
$\log p(\X | \Pm ) =^+ - \beta  \sum_{(i,j)\in \Edge} \|\Pm_{ij} - \X_i \X_j^\top \|^2_\mathrm{F} $, and $=^+$ denotes equality up to an additive constant.
\item The full posterior distribution: $p(\X|\Pm) \propto p(\Pm,\X)$.
\vspace{1pt}
\end{enumerate}
The MAP estimate is often easier to obtain and useful in practice. On the other hand, characterizing the full posterior can provide important additional information, such as \emph{uncertainty}; however, not surprisingly it is a much harder task. In addition to the usual difficulties associated with these tasks, in our context we are facing extra challenges due to the non-standard manifold of our latent variables.

\subsection{Maximum A-Posteriori Estimation}
The MAP estimation problem can be cast as a minimization problem on $\DP_n$, given as follows:
\begin{align*}
\X^\star = \argmin_{\X \in \DP_n^N} \Bigl\{ U(\X) :=  \sum\nolimits_{(i,j)\in \Edge} \|\Pm_{ij} - \X_i \X_j^\top \|^2_\mathrm{F} \Bigr\}
\end{align*}
where $U$ is called the potential energy function. We observe that the choice of the dispersion parameter has no effect on the MAP estimate. Although this optimization problem resembles conventional norm minimization, the fact that $\X$ lives in the cartesian product of Birkhoff polytopes renders the problem very complicated. 

Thanks to the retraction operator over the Birkhoff polytope (cf.\ Thm.~\ref{thm:retract}), we are able to use several Riemannian optimization algorithms~\cite{smith1994optimization}, without resorting to projection-based updates. In this study, we use the recently proposed Riemannian limited-memory BFGS (LR-BFGS)~\cite{huang2015broyden}, a powerful optimization technique that attains faster convergence rates by incorporating local geometric information in an efficient manner. This additional piece of information is obtained through an approximation of the inverse Hessian, which is computed on the most recent values of the past iterates with linear time- and space-complexity in the dimension of the problem. 
We give more detail on LR-BFGS in our supp. material. The detailed description of the algorithm can be found in \cite{huang2015broyden,yuan2016riemannian}. 

Finally, we round the resulting approximate solutions into a feasible one via Hungarian algorithm~\cite{munkres1957algorithms}, obtaining binary permutation matrices.



\subsection{Posterior Sampling via Riemannian Langevin Monte Carlo with Retractions}
In this section we will develop a Markov Chain Monte Carlo (MCMC) algorithm for generating samples from the posterior distribution $p(\X|\Pm)$, by borrowing ideas from \cite{simsekli2016stochastic,Liu2016,birdalSimsekli2018}. Once such samples are generated, we will be able to quantify the uncertainty in our estimation by using the generated samples.

The dimension and complexity of the Birkhoff manifold makes it very challenging to generate samples on $\DP_n$ or its product spaces and to the best of our knowledge there is no Riemannian MCMC algorithm that is capable of achieving this. There are existing Riemannian MCMC algorithms \cite{byrne2013,Liu2016}, which are able to draw samples on embedded manifolds; however, they require the exact exponential map to be analytically available, which in our case, can only be approximated by the retraction map at best. 

To this end, we develop an algorithmically simpler yet effective algorithm. Let the posterior density of interest be $\pi_{\cal H}(\X) := p(\X| \Pm) \propto \exp(-\beta U(\X))$ with respect to the Hausdorff measure. We then define an \emph{embedding} $\xi : \R^{N(n-1)^2} \mapsto \DP_n^N$ such that $\xi(\xe) = \X$ for $\xe \in \mathbb{R}^{N(n-1)^2}$. By the area formula (cf. Thm. 1 in \cite{diaconis2013sampling}), we have the following expression for the embedded posterior density $\pi_\lambda$ (with respect to the Lebesgue measure):
\begin{align}
\pi_{\cal H}(\thb) = \pi_\lambda(\xe) / \sqrt{|\mathbf{G}(\xe)|},
\end{align}
where $\mathbf{G}$ denotes the Riemann metric tensor. 



We then consider the following stochastic differential equation (SDE), which is a slight modification of the SDE that is used to develop the Riemannian Langevin Monte Carlo algorithm \cite{girolami2011riemann,xifara2014langevin,patterson2013stochastic}:
\begin{align*}
\mathrm{d} \xe_t = ( -\mathbf{G}^{-1} \nabla_{\xe}  U_\lambda(\xe_t) +\boldsymbol{\Gamma}_t ) \mathrm{d}t + \sqrt{2/\beta \mathbf{G}^{-1}} \mathrm{d} \mathrm{B}_t,
\end{align*}
where $\mathrm{B}_t$ denotes the standard Brownian motion and $\boldsymbol{\Gamma}_t$ is called the correction term that is defined as follows:
$[\boldsymbol{\Gamma}_t(\xe)]_i = \sum_{j=1}^{N(n-1)^2} {\partial [\boldsymbol{G}^{-1}_t (\xe)]_{ij}}/{\partial{\xe_j}}$.

By Thm.~1 of \cite{ma2015complete}, it is easy to show that the solution process $(\xe_t)_{t \geq 0}$ leaves the embedded posterior distribution $\pi_\lambda$ invariant. Informally, this result means that if we could exactly simulate the continuous-time process $(\xe_t)_{t \geq 0}$, the distribution of the sample paths would converge to the embedded posterior distribution $\pi_\lambda$, and therefore the distribution of $\xi(\xe_t)$ would converge to $\pi_{\cal H}(\X)$. However, unfortunately it is not possible to exactly simulate these paths and therefore we need to consult approximate algorithms. 

\insertimageC{1}{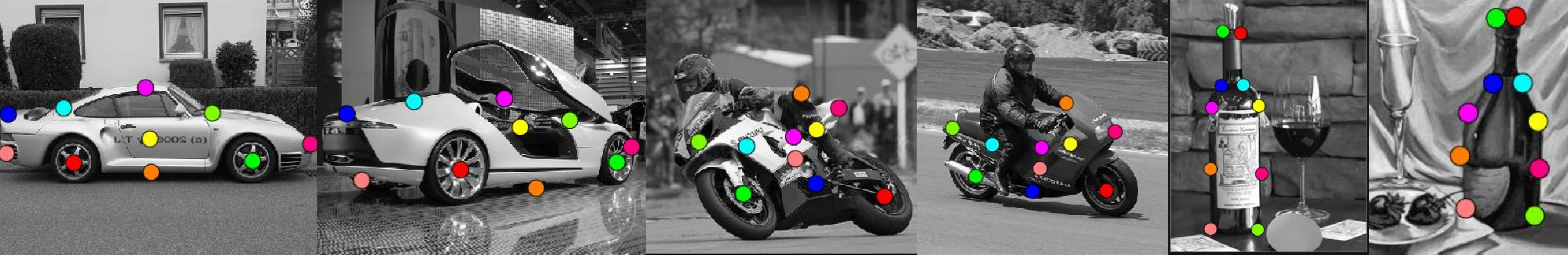}{Sample images and manually annotated correspondences from the challenging Willow dataset~\cite{cho2013learning}. Images are plotted in pairs (there are multiple) and in gray for better viewing.\vspace{-4mm}}{fig:willow}{t!}
A possible way to numerically simulate the SDE would be to use standard discretization tools, such as the Euler-Maruyama integrator \cite{chen2015convergence}. However, this would require knowing the analytical expression of $\xi$ and constructing $\mathbf{G}_t$ and $\boldsymbol{\Gamma}_t$ at each iteration. On the other hand, recent results have shown that we can simulate SDEs directly on their original manifolds by using geodesic integrators \cite{byrne2013,Liu2016,holbrook2018note}, which bypasses these issues altogether. Yet, these approaches require the exact exponential map of the manifold to be analytically available, restricting their applicability in our context. 

Inspired by the recent manifold optimization algorithms  \cite{tripuraneni2018averaging}, we propose to replace the exact, intractable exponential map arising in the geodesic integrator with the tractable retraction operator given in Thm.~\ref{thm:retract}. We develop our recursive scheme, we coin as \emph{retraction Euler integrator}:
\begin{align}
\mathbf{V}_i^{(k+1)} &= \Pi_{\X_i^{(k)}} ( h \nabla_{\X_i} U(\X_i^{(k)}) + \sqrt{2 h/\beta} \mathbf{Z}_i^{(k+1)} )\\ 
\X_i^{(k+1)}  &= R_{\X_i^{(k)}}  ( \mathbf{V}_i^{(k+1)}), \hspace{30pt} \forall i \in \{1,\dots,N\}
\end{align} 
where $h>0$ denotes the step-size, $k$ denotes the iterations, $\mathbf{Z}_i^{(k)}$ denotes standard Gaussian random variables in $\mathbb{R}^{n \times n}$, $\X_i^{(0)}$ denotes the initial absolute permutations. The derivation of this scheme is similar to \cite{Liu2016} and we provide more detailed information in the supplementary material. To the best of our knowledge, the convergence properties of the geodesic integrator that is approximated with a retraction operator have not yet been analyzed. We leave this analysis as a futurework, which is beyond the scope of this study.   

We note that the term $\|\X_{ij}\|^2_\mathrm{F}$ plays an important role in the overall algorithm since it prevents the latent variables $\X_i$ to go the extreme points of the Birkhoff polytope, where the retraction operator becomes inaccurate. We also note that, when $\beta \rightarrow \infty$, the distribution $\pi_{\cal H}$ concentrates on the global optimum $\X^\star$ and the proposed retraction Euler integrator becomes the Riemannian gradient descent with a retraction operator. 

\section{Experiments and Evaluations}
\label{sec:exp}
\subsection{Real Data}
\begin{table*}[htbp]
  \centering
  \caption{Our results on the \textit{WILLOW Object Class} graph matching dataset. Wang$^-$ refers to running Wang~\cite{wang2018multi} without the geometric consistency term. The vanilla version of our method, \textit{Ours}, already lacks this term. \textit{Ours-Geom} then refers to initializing Wang's verification method with our algorithm. For all the methods, we use the original implementation of the authors.}
  \setlength{\tabcolsep}{2.5pt}
  \resizebox{\textwidth}{!}
  {
    \begin{tabular}{lccccccccc}
    Dataset & Initial & Spectral~\cite{pachauri2013solving} & MatchLift~\cite{huang2013consistent} & MatchALS~\cite{zhou2015multi} & MatchEig~\cite{maset2017} & Wang$^-$~\cite{wang2018multi} & Ours  & Wang~\cite{wang2018multi} & Ours-Geom \\
    \midrule
    Car   & 0.48  & 0.55  & 0.65  & 0.69  & 0.66  & 0.72  & 0.71  & 1.00  & 1.00 \\
    Duck  & 0.43  & 0.59  & 0.56  & 0.59  & 0.56  & 0.63  & 0.67  & 0.93\tablefootnote{~\cite{wang2018multi} reports a value of $0.88$, but for their method, we attained $0.93$ and therefore report this value.}  & 0.96 \\
    Face  & 0.86  & 0.92  & 0.94  & 0.93  & 0.93  & 0.95  & 0.95  & 1.00  & 1.00 \\
    Motorbike & 0.30  & 0.25  & 0.27  & 0.34  & 0.28  & 0.40  & 0.37  & 1.00  & 1.00 \\
    Winebottle & 0.52  & 0.64  & 0.72  & 0.70  & 0.71  & 0.73  & 0.73  & 1.00  & 1.00 \\
    CMU-House & 0.68  & 0.90  & 0.94  & 0.92  & 0.94  & 0.98  & 0.98  & 1.00  & 1.00 \\
    CMU-Hotel & 0.64  & 0.81  & 0.87  & 0.86  & 0.92  & 0.94  & 0.96  & 1.00  & 1.00 \\
    \midrule
    Average & 0.52  & 0.59  & 0.65  & 0.66  & 0.71  & 0.76  & \textbf{0.77} & 0.99  & \textbf{0.99} \\
    \end{tabular}%
    }\vspace{-3mm}
  \label{tab:willow}%
\end{table*}
\paragraph{2D Multi-image Matching}
We run our method to perform multiway graph matching on two datasets, CMU~\cite{caetano2009learning} and Willow Object Class~\cite{cho2013learning}. CMU is composed of House and Hotel objects viewed under constant illumination and smooth motion. Initial pairwise correspondences as well as ground truth (GT) absolute mappings are provided within the dataset.
Object images in Willow dataset include pose, lighting, instance and environment variation as shown in Fig.~\ref{fig:willow}, rendering naive template matching infeasible. For our evaluations, we follow the same design as Wang~\etal~\cite{wang2018multi}. We first extract local features from a set of $227\times 227$ patches centered around the annotated landmarks, using the prosperous Alexnet~\cite{krizhevsky2012imagenet} pretrained on ImageNet~\cite{deng2009imagenet}. Our descriptors correspond to the feature map responses of Conv4 and Conv5 layers anchored on the hand annotated keypoints. These features are then matched by the Hungarian algorithm~\cite{munkres1957algorithms} to obtain initial pairwise permutation matrices $\Pm_0$.

We initialize our algorithm by the closed form MatchEIG~\cite{maset2017} and evaluate it against the state of the art methods of Spectral~\cite{pachauri2013solving}, MatchALS~\cite{zhou2015multi}, MatchLift~\cite{huang2013consistent}, MatchEIG~\cite{maset2017}, and Wang~\etal~\cite{wang2018multi}. The size of the universe is set to the number of features per image. We assume that this number is fixed and partial matches are not present. Handling partialities while using the Birkhoff structure is left as a future work. Note that~\cite{wang2018multi} uses a similar cost function to ours in order to initialize an alternating procedure that in addition exploits the geometry of image coordinates. Authors also use this term as an extra bit of information during their initialization. The standard evaluation metric, \textit{recall}, is defined over the pairwise permutations as:
\begin{align}
\label{eq:eval}
R(\{\hat{\Pm}_i\} | \Pm^{\text{gnd}})= \frac{1}{n|\Edge|}\sum\limits_{(i,j)\in \Edge} {\Pm^{\text{gnd}}_{ij}\odot(\hat{\Pm}_i\hat{\Pm}_j^\top)}
\end{align}
where $\Pm^{\text{gnd}}_{ij}$ are the GT relative transformations and $\hat{\Pm}_i$ is an estimated permutation. $R=0$ in the case of no correctly found correspondences and $R=1$ for a perfect solution.
Tab.~\ref{tab:willow} shows the results of different algorithms as well as ours. Note that our Birkhoff-LRBFGS method that operates solely on pairwise permutations outperforms all methods, even the ones which make use of geometry during initialization. Moreover, when our method is used to initialize Wang~\etal~\cite{wang2018multi} and perform geometric optimization, we attain the top results. These findings validate that walking on the Birkhoff Polytope, even approximately, and using Riemannian line-search algorithms constitute a promising direction for optimizing the problem at hand.

\vspace{-4mm}
\paragraph{Uncertainty Estimation in Real Data} We now run our confidence estimator on the same Willow Object Class~\cite{cho2013learning}. To do that, we first find the optimal point where synchronization is at its best. Then, we set $h \gets 0.0001$, $\beta \gets [0.075,0.1]$ and automatically start sampling the posterior around this mode for $1000$ iterations. Note that $\beta$ is a critical parameter which can also be dynamically controlled~\cite{birdalSimsekli2018}. Larger values of $\beta$ cannot provide enough variation for a good diversity of solutions. Smaller values cause greater random perturbations leading to samples far from the optimum. This can cause divergence or samples not explaining the local mode. Nevertheless, all our tests worked well with values in the given range.

The generated samples are useful in many applications, e.g. fitting distributions or providing additional solution hints. We address the case of multiple hypotheses generation for the permutation synchronization problem and show that generating an additional per-edge candidate with high certainty helps to improve the recall. Tab.~\ref{tab:uncertainty} shows the top-K scores we achieve by simply incorporating $K$ likely samples. Note that, when 2 matches are drawn at random and contribute as potentially correct matches, the recall is increased only by $2\%$, whereas including our samples instead boosts the multi-way matching by $6\%$.
\begin{table}[htbp]
  \centering
  \caption{Using \textbf{top-K} errors to rank by uncertainty. Based on the confidence information we could retain multiple hypotheses. This is not possible by the other approaches such as Wang~\etal~\cite{maset2017,wang2018multi}. \textit{Rand-K} refers to using $K-1$ additional random hypotheses to complement the found solution. \textit{Ours-K} ranks assignments by our probabilistic certainty and retains top-K candidates per point.}
  \setlength{\tabcolsep}{2.75pt}
  \resizebox{\columnwidth}{!} {
    \begin{tabular}{lcccccc}
    Dataset & Wang  & Ours  & Rand-2 & Ours-2 & Rand-3 & Ours-3 \\
    \midrule
    Car   & 0.72  & 0.71  & 0.73  & 0.76  & 0.76  & 0.81 \\
    Duck  & 0.63  & 0.67  & 0.67  & 0.69  & 0.67  & 0.72 \\
    Face  & 0.95  & 0.95  & 0.96  & 0.97  & 0.96  & 0.98 \\
    Motorbike & 0.40  & 0.37  & 0.45  & 0.49  & 0.52  & 0.60 \\
    Winebottle & 0.73  & 0.74  & 0.77  & 0.82  & 0.79  & 0.85 \\
    \midrule
    Avg. & 0.69  & 0.69  & 0.71  & \textbf{0.75} & 0.74  & \textbf{0.79} \\
    \end{tabular}%
    }
  \label{tab:uncertainty}%
  \vspace{-3mm}
\end{table}%

We further present illustrative results for our confidence prediction in Fig.~\ref{fig:uncertainty}. There, unsatisfactory solutions arising in certain cases are improved by analyzing the uncertainty map. The column (e) of the figure depicts the top-2 assignments retained in the confidence map and (e) plots the assignments that overlap with the true solution. Note that, we might not have access to such an oracle in real applications and only show this to illustrate potential use cases of the estimated confidence map.

\subsection{Evaluations on Synthetic Data}
We synthetically generate $28$ different problems with varying sizes: $M\in[10,100]$ nodes and $n\in[16,100]$ points in each node. For the scenario of image matching, this would correspond to $M$ cameras and $N$ features in each image. We then introduce $15\%-35\%$ random swaps to the GT absolute permutations and compute the \textit{observed} relative ones. Details of this dataset are given in suppl. material. Among all $28$ sets of synthetic data, we attain an overall recall of $91\%$ whereas MatchEIG~\cite{maset2017} remains about $83\%$. 
\insertimageStar{1}{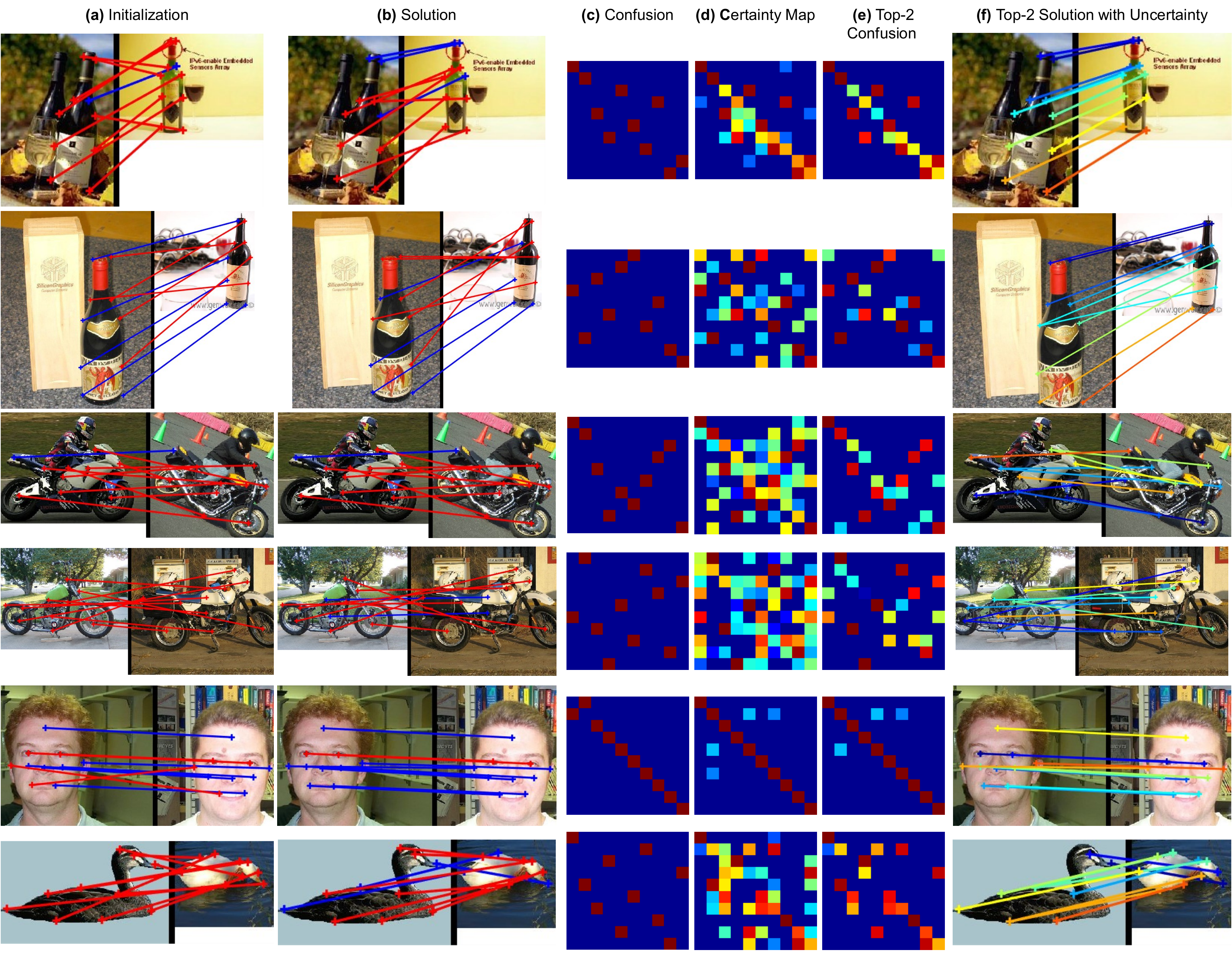}{Results from our confidence estimation. Given potentially erroneous solutions (\textbf{b}) to the problems initialized as in (\textbf{a}), our latent samples discover the uncertain assignments as shown in the middle three columns (\textbf{c}-\textbf{e}). When multiple top-2 solutions are accepted as potential positives, our method can suggest high quality hypotheses (\textbf{f}). The edges in the last column (\textbf{f}) is colored by their confidence value. Note that even though, for the sake of space we show pairs of images, the datasets contain multiple sets of images.\vspace{-3.0mm}}{fig:uncertainty}{t!}
 \vspace{-5mm}
\paragraph{Runtime Analysis} \hspace{5pt} Next, we assess the computational cost of our algorithm against the state of the art methods, on the dataset explained above. All of our experiments are run on a MacBook computer with an Intel i7 2.8GhZ CPU. Our implementation uses a modified Ceres solver~\cite{ceres-solver}. All the other algorithms use highly vectorized MATLAB 2017b code making our comparisons reasonably fair. Fig.~\ref{fig:runtime} tabulates runtimes for different methods excluding initialization. \textit{MatchLift} easily took more than 20min. for moderate problems and hence we choose to exclude it from this evaluation. It is noticeable that thanks to the ability of using more advanced solvers such as LBFGS, our method converges much faster than Wang~\etal and runs on par with the fastest yet least accurate spectral synchronization~\cite{pachauri2013solving}. 
The worst case theoretical computational complexity of our algorithm is $O_{\text{B-LRBFGS}} := O(K |\Edge| K_S (n^2+(2n)^3)$ where $K$ is the number of LBFGS iterations and $K_S$ the number of Sinkhorn iterations. While $K_S$ can be a bottleneck, in practice our matrices are already restricted to the Birkhoff manifold and Sinkhorn early-terminates, letting $K_S$ remain small. The complexity is: (1) linearly-dependent upon the number of edges, which in the worst case relates quadratically to the number of images $|\Edge|=N(N-1)$, (2) cubically dependent on $n$. This is due to the fact that projection onto the tangent space solves a system of $2n\times 2n$ equations.

\section{Conclusion}
In this work we have proposed two new frameworks for relaxed permutation synchronization on the manifold of doubly stochastic matrices. Our novel model and formulation paved the way to using sophisticated optimizers such as Riemannian limited-memory BFGS. We further integrated a manifold-MCMC scheme enabling posterior sampling and thereby confidence estimation. 
We have shown that our confidence maps are informative about cycle inconsistencies and can lead to new solution hypotheses. We used these hypotheses in a top-K evaluation and illustrated its benefits. In the future, we plan to (i) address partial permutations, the inner region of the Birkhoff Polytope (ii) investigate more sophisticated MCMC schemes such as \cite{durmus2016stochastic,holbrook2018note,csimvsekli2017fractional,Liu2016,csimcsekli2018asynchronous} (iii) seek better use cases for our confidence estimates such as outlier removal.
\insertimageC{0.9}{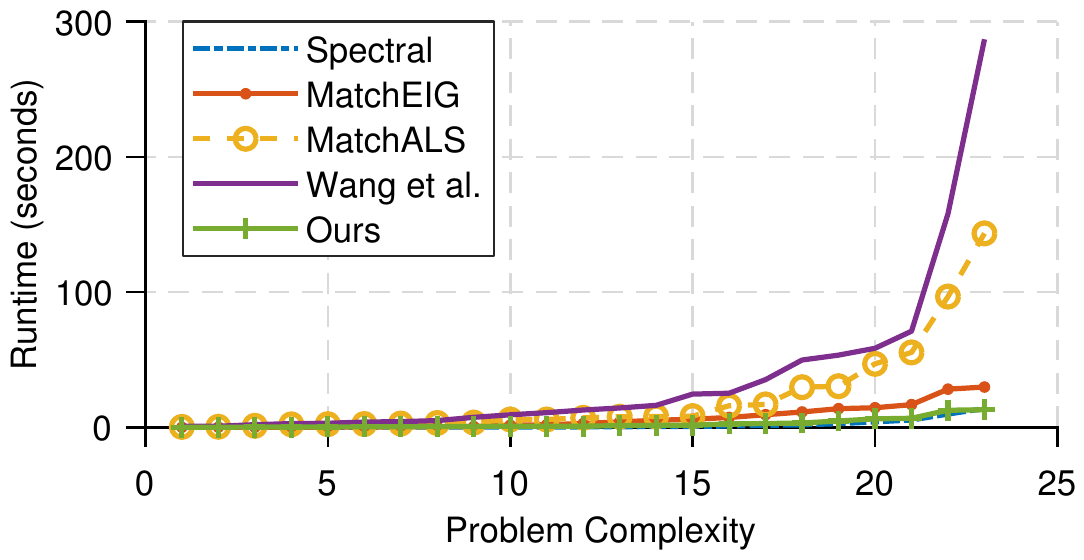}{Running times of different methods with increasing problem size: $N\in[10,100]$ and $n\in[16,100]$.\vspace{-6.5mm}}{fig:runtime}{t!}

{
\paragraph{Acknowledgements} Supported by the grant ANR-16-CE23-0014 (FBIMATRIX). Authors thank Haowen Deng for the initial 3D correspondences and, Benjamin Busam and Jesus Briales for fruitful discussions.
}

{
\bibliographystyle{ieee_fullname}

}

\onecolumn
\setcounter{section}{0}
\renewcommand\thesection{\Alph{section}}
\newcommand{\suppsection}{\subsection}
\clearpage
\begin{center}
\textbf{\large Probabilistic Permutation Synchronization using the Riemannian Structure of the Birkhoff Polytope - Supplementary Material}
\end{center}
\makeatletter

This part supplements our main paper by providing further algorithmic details on RL-BFGS, proofs of the propositions presented in the paper, derivations of the \textit{retraction Euler integrator} and additional experiments.

\section{Riemannian Limited Memory BFGS}
We now explain the R-LBFGS optimizer used in our work. To begin with, we recall the foundational BFGS~\cite{fletcher2013practical}, that is a quasi Newton method operating in the Euclidean space. We then review the simple Riemannian descent algorithms that employ line-search. Finally, we explain the R-BFGS used to solve our synchronization problem. R-BFGS can be modified slightly to arrive at the desired limited memory Riemannian BFGS solver.

\subsection{Euclidean BFGS}
The idea is to approximate the true hessian $\mathbf{H}$ with ${\bm{\B}}$, using updates specified by the gradient evaluations\footnote{Note that certain implementations can instead opt to approximate the inverse Hessian for computational reasons.}. We will then transition from this Euclidean space line search method to Riemannian space optimizers. For clarity, in Alg. 1 we summarize the Euclidean BFGS algorithm, that computes ${\bm{\B}}$ by using the most recent values of the past iterates. Note that many strategies exist to initialize ${\bm{\B}}_0$, while a common choice is the scaled identity matrix ${\bm{\B}}_0=\gamma\mathbf{I}$. Eq.~\ref{eq:bfgs} corresponds to the particular BFGS-update rule. In the limited-memory successor of BFGS, the L-BFGS, the Hessian matrix $\mathbf{H}$ is instead approximated up to a pre-specified rank in order to achieve linear time and space-complexity in the dimension of the problem. 

\begin{algorithm2e} [h!]
\DontPrintSemicolon
\SetKwInOut{Input}{input}\Input{A real-valued, differentiable potential energy $U$, initial iterate $\X_0$ and initial Hessian approximation ${\bm{\B}}_0$.}
$k \gets 0$\\
 \While{$\x_{k}$ {does not sufficiently minimize} $f$}{
   Compute the direction $\bm{\eta}_k$ by solving ${\bm{\B}}_k\bm{\eta}_k=-\nabla U(\x_k)$ for $\bm{\eta}_k$.\\
   Define the new iterate $\x_{k+1}\gets \x_k+\bm{\eta}_k$.\\
   Set $\s_k \gets \x_{k+1}-\x_{k}$ and $\y_k \gets \nabla U(\x_{k+1}) - \nabla U(\x_{k})$.\\
   Compute the new Hessian approximation:
   \begin{equation}
   \label{eq:bfgs}
   {\bm{\B}}_{k+1} = {\bm{\B}}_k + \frac{\y_k\y_k^\top}{\y_k^\top\s_k} - \frac{{\bm{\B}}_k\s_k\s_k^\top\bm{\B}_k}{\s_k^\top {\bm{\B}}_k \s_k}.
   \end{equation}
   $k\gets k+1$.
   }
\caption{Euclidean BFGS}
\label{algo:bfgs}
\end{algorithm2e}
 
\subsection{Riemannian Descent and Line Search}
\label{sec:rlnsearch}
The standard descent minimizers can be extended to operate on Riemannian manifolds $\Man$ using the geometry of the parameters. A typical Riemannian update can be characterized as:
\begin{align}
\x_{k+1} = R_{\x_k}(\tau_k \bm{\eta}_k)
\end{align}
where $R$ is the retraction operator, i.e. the smooth map from the tangent bundle $\T\Man$ to $\Man$ and $R_{\x_k}$ is the restriction of $R$ to $\T_{\x_k}$, the tangent space of the current iterate $\x_k$, $R_{\x_k}:\T_{\x_k}\rightarrow \Man$. The descent direction is defined to be on the tangent space $\bm{\eta}_k\in \T_{\x_k}\Man$. When manifolds are rather simple shapes, the size of the step $\tau_k$ can be a fixed value. However, for most matrix manifolds, some form of a line search is preferred to compute $\tau_k$. The retraction operator is used to take steps on the manifold and is usually derived analytically. When this analytic map is length-preserving, it is called \textit{true} exponential map. However, such exactness is not a requirement for the optimizer and as it happens in the case of doubly stochastic matrices, $R_{\x_k}$ only needs to be an approximate \textit{retraction}, e.g. first or second order. In fact, for a map to be valid retraction, it is sufficient to satisfy the following conditions:

\begin{wrapfigure}[1]{r}{0.27\textwidth}
    \vspace{-80pt}
    \centering
    \includegraphics[height=0.2\textwidth]{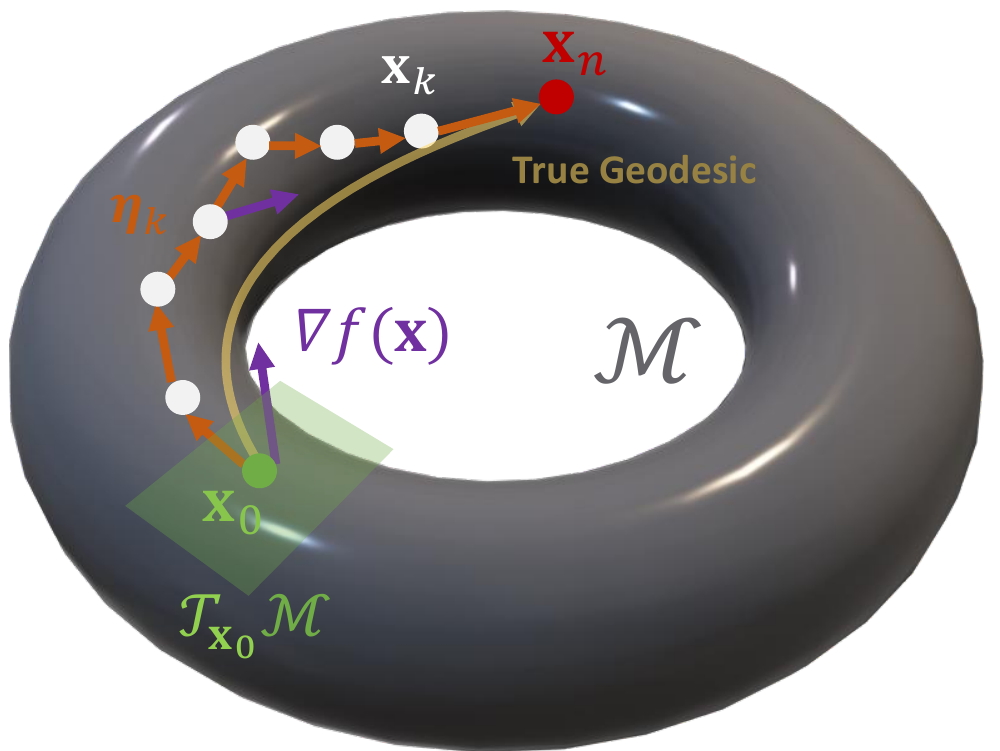}
    \caption{\small Visualization of the entities used on a sample toroidal manifold.}
    \label{fig:torus}
\end{wrapfigure} 

\begin{minipage}{.65\textwidth}
\begin{enumerate}\itemsep0em
\item $R$ is continuously differentiable.
\item $R_{\x}(\zero_{\x})=\x$, where $\zero_{\x}$ denotes the zero element of $\T_{\x}\Man$. This is called the \textit{centering} property.
\item The curve $\gamma_{\bm{\eta}_{\x}}(\tau)=R_{\x}(\tau \bm{\eta}_{\x})$ satisfies:
\begin{align}
\frac{d\gamma_{\bm{\eta}_{\x}}(\tau)}{d\tau}\Bigr\rvert_{\tau = 0}=\bm{\eta}_{\x}, \, \forall \bm{\eta}_{\x} \in \T_{\x}\Man.
\end{align}
This is called the \textit{local rigidity} property.
\end{enumerate}
\vspace{12pt}
\end{minipage}

Considering all these, we provide, in Alg. 2, a general Riemannian descent algorithm, which can be customized by the choice of the direction, retraction and the means to compute the step size. Such a concept of minimizing by walking on the manifold is visualized in Fig.~\ref{fig:torus}.

 \begin{algorithm2e} [h!]
 \DontPrintSemicolon
 \SetKwInOut{Input}{input}
 \Input{A Riemannian manifold $\Man$, a retraction operator $R$ and initial iterate $\x_k\in\Man$ where $k=0$.}
 \While{$\x_{k}$ \text{does not sufficiently minimize} $f$}{
    Pick a gradient related descent direction $\bm{\eta}_k\in\T_{\x_k}\Man$.\\
    Choose a retraction $R_{\x_k}:\T_{\x_k}\Man \rightarrow \Man$.\\
    Choose a step length $\tau_k\in \mathbb{R}$.\\
    Set $\x_{k+1}\gets R_{\x_k}(\tau_k\bm{\eta}_k)$.\\
	$k\gets k+1$.
    }
 \caption{General Riemannian Line Search Minimizer}
 \label{algo:linesearch}
 \end{algorithm2e}

It is possible to develop new minimizers by making particular choices for the Riemannian operators in Alg 2. We now review the Armijo variant of the Riemannian gradient descent~\cite{hosseini2018line}, that is a common and probably the simplest choice for an accelerated optimizer on the manifold. Though, many other line-search conditions such as Barzilai-Borwein~\cite{iannazzo2017riemannian} or strong Wolfe~\cite{ring2012optimization} can be used. The pseudocode for this backtracking version is given in Alg. 3. Note that the Riemannian gradient $\text{grad}f(\x)$ is simply obtained by projecting the Euclidean gradient $\nabla f(\x)$ onto the manifold $\Man$ and the next iterate is obtained through a line search so as to satisfy the Armijo condition~\cite{armijo1966minimization}, tweaked to use the retraction $R$ for taking steps. For some predefined Armijo step size, this algorithm
is guaranteed to converge for all retractions~\cite{absil2009optimization}. 
 
 \begin{algorithm2e} [h!]
 \DontPrintSemicolon
 \SetKwInOut{Input}{input}
 \Input{A Riemannian manifold $\Man$, a retraction operator $R$, the projection operator onto the tangent space $\Pi_{\x_k}:\mathbb{R}^n\rightarrow\T_{\x_k}\Man$, a real-valued, differentiable potential energy $f$, initial iterate $\x_0\in\Man$ and the Armijo line search scalars including $c$.}
 \While{$\x_{k}$ \text{does not sufficiently minimize} $f$}{
 {\color{purple} \small \tcp{Euclidean gradient to Riemannian direction}}
    $\bm{\eta}_k \gets -\text{grad}f(\x_k) \triangleq \Pi_{\x_k} (-\nabla f(\x_k))$.\\
    Select $\x_{k+1}$ such that:
	\begin{equation}    
    f(\x_k)-f(\x_{k+1})\geq c\big( f(\x_k) - f(R_\x (\tau_k \bm{\eta}_k))\big),
	\end{equation}\\
	where $\tau_k$ is the Armijo step size.\\
	$k\gets k+1$.
    }
 \caption{General Riemannian Steepest Descent with Armijo Line Search}
 \label{algo:armijo}
 \end{algorithm2e}

\subsection{LR-BFGS for Minimization on $\DP_n$}
Based upon the ideas developed up to this point, we present the Riemannian variant of the L-BFGS. The algorithm is mainly based on Huang~\etal and we refer the reader to their seminal work for more details~\cite{huang2015broyden}. It is also worth noting~\cite{qi2010riemannian}. Similar to the Euclidean case, we begin by summarizing a BFGS algorithm with the difference that it is suited to solving the synchronization task. This time, $\bm{\B}$ will approximate the action of the Hessian on the tangent space $\T_{\x_k}\Man$. Generalizing any (quasi-)Newton method to the Riemannian setting thus requires computing the Riemannian Hessian operator, or its approximation. This necessitates taking a some sort of a directional derivative of a vector field. As the vector fields belonging to different tangent spaces cannot be compared immediately, one needs the notion of connection $\Gamma$ generalizing the directional derivative of a vector field. This connection is closely related to the concept of vector transport $T: \T\Man \otimes \T\Man \rightarrow \T \Man$ which allows moving from one tangent space to the other $T_{\bm{\eta}}(\bm{\zeta}): (\bm{\eta},\bm{\zeta})\in\T_\x\Man \rightarrow \T_{R_{\x}(\bm{\eta})}\Man \,$. For the first order treatment of the Birkhoff Polytope, this vector transport takes a simple form:
\begin{equation}
T_{\bm{\eta}}(\zeta) \in \T_{R_{\x}(\bm{\eta})}\DP_n \triangleq \Pi_{R_{\x}(\bm{\eta})}(\bm{\zeta}).
\end{equation}
where $\Pi_{\X}(\cdot)$ is the projection onto the tangent space of $\X$ as defined in the main paper. We give the pseudocode of the R-BFGS algorithm in Alg. 4 below.
To ensure Riemannian L-BFGS always produces a descent direction, it is necessary to adapt a line-search algorithm which satisfies strong Wolfe (SW) conditions~\cite{wolfe1969convergence,wolfe1971convergence}. Roptlib~\cite{HAG2016} does implement the SW while ManOpt~\cite{manopt} uses the simpler Armijo conditions, even for the LR-BFGS. Another simple possibility is to take the steps on the Manifold using the retraction, while performing the line search in the Euclidean space. This results in a projected, approximate line search, but can terminate quite quickly, making it possible to exploit existing Euclidean space SW LBFGS solvers such as Ceres~\cite{ceres-solver}. Without delving into rigorous proofs, we provide one such implementation at {\url{github.com/tolgabirdal/MatrixManifoldsInCeres}} where several matrix manifolds are considered.
We leave the analysis of such approximations for a future study and note that the results in the paper are generated by limited memory form of the R-BFGS procedure summarized under Alg. 4. While many modifications do exist, LR-BFGS, in essence, is an approximation to R-BFGS, where $O(MN^2)$ storage of the full Hessian is avoided by unrolling the RBFGS descent computation and using the last $m$ input and gradient differences where $m<<MN^2$. This allows us to handle large data regimes.

 \begin{algorithm2e} [h!]
 \DontPrintSemicolon
 \SetKwInOut{Input}{input}
 \Input{Birkhoff Polytope $\DP$ with Riemannian (Fisher information) metric $g$, first order retraction $R$, the parallel transport $T$, a real-valued, differentiable potential energy function of synchronization $U$, initial iterate $\X_0$ and initial Hessian approximation ${\bm{\B}}_0$.}
 \While{$\x_{k}$ {does not sufficiently minimize} $f$}{
 	Compute the Euclidean gradients using:
 	\begin{align}
 	\nabla_{\X_i} U(\X) &= \sum\limits_{(i,j)\in E} -2 (\Pm_{ij}-\X_i\X_j^\top)\X_j \qquad
 	 \nabla_{\X_j} U(\X) = \sum\limits_{(i,j)\in E} -2 (\Pm_{ij}-\X_i\X_j^\top)\X_i
 	\end{align}\\
 	Compute the Riemannian gradient\,\, $\text{grad }U(\X_k) \gets \triangleq \Pi_{\X_k} (-\nabla U(\X_k))$.\\
    Compute the direction $\bm{\eta}_k\in \T_{\X_k}\DP_n$ by solving ${\bm{\B}}_k\bm{\eta}_k=-\text{grad }U(\X_k)$.\\
    Given $\bm{\eta}_k$, execute a line search satisfying the strong Wolfe conditions~\cite{wolfe1969convergence,wolfe1971convergence,huang2015broyden}. Set step size $\tau_k$. \\
    Set $\x_{k+1}=R_{\x_k}(\tau_k\bm{\eta}_k)$.\\
    Use vector transport to define:
    \begin{align}
    \sing_k = T_{\tau_k\bm{\eta}_k}(\tau_k\bm{\eta}_k)\,,\qquad
    \Y_k = \text{grad }U(\X_{k+1})-T_{\tau_k\bm{\eta}_k}(\text{grad }U(\x_k)).
    \end{align}\\
	Compute $\tilde{\bm{\B}_k}=T_{\tau_k\bm{\eta}_k}\circ \bm{\B}_k \circ T_{\tau_k \bm{\eta}_k}^{+}$ where $T^{+}$ is denotes (pseudo-)inverse of the transport.\\
    Compute the linear operator $\bm{\B}_{k+1}:\T_{\x_{k+1}}\Man\rightarrow \T_{\x_{k+1}}\Man$:
   \begin{equation}
   \label{eq:rbfgs}
   \tilde{{\bm{\B}}}_{k+1}\Z = {\tilde{\bm{\B}}}_k\Z + \frac{g(\Y_k,\Z)}{g(\y_k,\sing_k)}\Y_k - \frac{g(\sing_k,\tilde{\bm{\B}}\Z)}{g(\sing_k, \tilde{\bm{\B}}\sing_k)} \qquad \forall \Z \in \T_{\x_{k+1}}\DP_n.
   \end{equation}\\
   $k\gets k+1$.
 }
 \caption{Riemannian BFGS for Synchronization on $\DP_n$}
 \label{algo:lbfgs}
 \end{algorithm2e}
\section{Proof of Proposition 1 and Further Analysis}
\begin{proof}
Consider the ray cast outwards from the origin: $\mathbf{r} = \mathbb{R}_{\geq 0} \mathbbm{1}$. $\mathbf{r}$ exits $\DP_n$ at $\tfrac{1}{n} \mathbbm{1}$ but exits the sphere $\Scal^{n-1}$ at $\tfrac{1}{\sqrt{n}} \mathbbm{1}$. This creates a gap between the two manifolds whose ratio grows to $\infty$ as $n$ grows. For a perspective of optimization, consider the linear function $\lambda(\Pm) = \sum_{i,j} P_{ij}$. $\lambda(\Pm)$ is minimized on $\DP_n$ at $\tfrac{1}{n} \mathbbm{1}$ and on the sphere at $\tfrac{1}{\sqrt{n}} \mathbbm{1}$,
\end{proof}

We now look at the restricted orthogonal matrices and seek to find the difference between optimizing a linear functional on them and $\DP_n$. Note that minimizing functionals is ultimately what we are interested in as the problems we consider here are formulated as optimization. Let $\mathcal{A}$ be the affine linear space of $(n-1)\times(n-1)$ matrices whose rows and columns sum to 1, \ie doubly stochastic but with no condition on the signs or \textit{a generalized doubly stochastic matrix}. The Birkhoff Polytope $\DP_n$ is contained in $\mathcal{A}$, whereas the orthogonal group $\Og_n$ is not. So, there exists an affine functional $\lambda$ that is minimized at a point on $\DP_n$, $\lambda=0$ but not on $\Og$. Let $\mathcal{AO}_n=\DP_n \cap \Og_n$, a further restricted manifold. This time unlike the case of $\DP_n$, $\mathcal{AO}_n$ would not coincide the permutations $\mathcal{P}$ due to the negative elements. In fact $\mathcal{P}\subset\mathcal{AO}_n$.
\setcounter{prop}{2}
\begin{prop}
The ratio between the time a ray leaves $\DP_n$ and the same ray leaves $\mathcal{AO}_n$ can be as large as $n-1$.
\end{prop}
\begin{proof}
Consider the line through $\frac{1}{n}\mathbbm{1}$ and $\Id$ : $l(x) = \frac{1+x}{n}\mathbbm{1} - x\Id$. Such a ray leaves $\DP_n$ at $x=\frac{1}{n-1}$ and for all $x\in [-1,1]$ is in the convex hull of $\mathcal{AO}_n$. When $x=1$, it is an orthogonal matrix, \ie on $\Og_n$. Hence, the ray can be on $\mathcal{AO}_n$ for $n-1$ times as long as in $\DP_n$. This quantity also grows to infinity as $n\rightarrow \infty$. If same analysis is done for the case of the sphere, the ratio is found to be $(n-1)^{3/2}$, grows quicker to infinity and is a larger quantity.
\end{proof}

\section{Proof of Proposition 2}

\begin{proof}
The conclusion of the proposition states that $p(\X)$ and $p(\Pm|\X)$ have the following form:
\begin{align}
    p(\X) =& \frac1{C} \exp \Bigl(-\beta  \sum\limits_{(i,j)\in \Edge}\|\X_{ij}\|_{\mathrm{F}}^2 \Bigr)  \prod\limits_{(i,j) \in \Edge}  Z_{ij} \\
p(\Pm | \X) =&  \prod_{(i,j)\in \Edge} p(\Pm_{ij} | \X_i, \X_j) \\
=& \prod_{(i,j)\in \Edge} \exp\Bigl( 2\beta\>\mathrm{tr}(\Pm_{ij}^\top \X_{ij}) \Bigr)  \frac1{Z_{ij}} \\
=& \exp\Bigl( \beta \sum_{(i,j)\in \Edge} 2 \mathrm{tr}(\Pm_{ij}^\top \X_{ij}) \Bigr) \prod_{(i,j)\in \Edge}\frac1{Z_{ij}}
\end{align}
where 
\begin{align}
    Z_{ij} :=  \prod\limits_{b=1}^{B_{ij}} Z_\X(\beta, \theta_{ij,b}). \label{eqn:normconst}
\end{align}
Our goal is to verify that the following holds with the above definitions:
\begin{align}
    p(\Pm,\X) = \frac1{Z} \exp\Bigl(-\beta \sum_{(i,j)\in \Edge} \|\Pm_{ij}- \X_i \X_j^\top \|_{\mathrm{F}}^2 \Bigr) = \frac1{Z} \prod_{(i,j)\in \Edge} \psi(\Pm_{ij},\X_i,\X_j),
\end{align}
where $Z$ is a positive constant (cf.\ Section 4 in the main paper). Then, we can easily verify this equality as follows: 
\begin{align}
    p(\Pm,\X) &= p(\Pm|\X) p(\X) \\
    &=  \frac1{C} \exp \Bigl(-\beta  \sum\limits_{(i,j)\in \Edge}\|\X_{ij}\|_{\mathrm{F}}^2 \Bigr)   \exp\Bigl( \beta \sum_{(i,j)\in \Edge} 2 \mathrm{tr}(\Pm_{ij}^\top \X_{ij}) \Bigr) \\
    &=  \frac1{C} \exp \Bigl(-\beta   \sum\limits_{(i,j)\in \Edge} \bigl(\|\X_{ij}\|_{\mathrm{F}}^2 - 2 \mathrm{tr}(\Pm_{ij}^\top \X_{ij}) \bigr) \Bigr)  \\
    &= \frac1{C} \exp \Bigl(-\beta   \sum\limits_{(i,j)\in \Edge}\bigl(\|\X_{ij}\|_{\mathrm{F}}^2 - 2 \mathrm{tr}(\Pm_{ij}^\top \X_{ij}) - n + n \bigr) \Bigr) \\
    &= \frac1{C} \exp(\beta n |\Edge|) \exp \Bigl(-\beta  \sum\limits_{(i,j)\in \Edge}\bigl( \|\X_{ij}\|_{\mathrm{F}}^2 - 2 \mathrm{tr}(\Pm_{ij}^\top \X_{ij}) + \|\Pm_{ij}\|_{\mathrm{F}}^2 \bigr) \Bigr) \label{eqn:interm} \\
    &= \frac1{Z}  \exp \Bigl(-\beta  \sum\limits_{(i,j)\in \Edge} \|\Pm_{ij}- \X_i \X_j^\top \|_{\mathrm{F}}^2 \Bigr)
\end{align}
where we used the fact that $\|\Pm_{ij}\|_{\mathrm{F}}^2 = n$ in Equation~\ref{eqn:interm} since $\Pm_{ij} \in \mathcal{P}_n$ and $Z = C \exp(-\beta n |\Edge|)$. 

\vspace{5pt}

In the rest of the proof, we will characterize the normalizing constant $Z_{ij} = \int_{\mathcal{P}_n} \exp(2\beta \mathrm{tr}(\Pm_{ij}^\top \X_{ij}) \mathrm{d} \Pm_{ij}$. Here, we denote the the counting measure on $\mathcal{P}_n$ as $\mathrm{d} \Pm_{ij}$.

\vspace{5pt} 

We start by decomposing $\X_{ij}$ via Birkhoff-von Neumann theorem:
\begin{align}
\X_{ij} = \sum\limits_{b=1}^{B_{ij}} \theta_{ij,b} \mathbf{M}_{ij,b}, \quad \sum\limits_{b=1}^{B_{ij}} \theta_{ij,b} =1, \quad B_{ij} \in \mathbb{Z}_+, \quad \theta_{ij,b} \geq 0, \> \mathbf{M}_{ij,b} \in \mathcal{P}_n, \quad \forall b = 1,\dots,B_{ij}.
\end{align}
Then, we have:
\begin{align}
    Z_{ij} &= \int_{\mathcal{P}_n} \exp\Bigl(2\beta \mathrm{tr}(\Pm_{ij}^\top \X_{ij} \Bigr) \mathrm{d} \Pm_{ij} \\
    &= \int_{\mathcal{P}_n} \exp\Bigl(2\beta \mathrm{tr}(\Pm_{ij}^\top \sum\limits_{b=1}^{B_{ij}} \theta_{ij,b} \mathbf{M}_{ij,b}  \Bigr) \mathrm{d} \Pm_{ij} \\
    &= \int_{\mathcal{P}_n} \exp\Bigl(2\beta \sum_{b=1}^{B_{ij}} \theta_{ij,b}  \mathrm{tr}(\Pm_{ij}^\top \mathbf{M}_{ij,b}) \Bigr) \mathrm{d} \Pm_{ij} \\
    & \geq \exp\Bigl( \int_{\mathcal{P}_n} 2\beta \sum_{b=1}^{B_{ij}} \theta_{ij,b}  \mathrm{tr}(\Pm_{ij}^\top \mathbf{M}_{ij,b})  \mathrm{d} \Pm_{ij} \Bigr) \label{eqn:jensen} \\
    & = \exp\Bigl(  \sum_{b=1}^{B_{ij}} \int_{\mathcal{P}_n} 2\beta  \theta_{ij,b}  \mathrm{tr}(\Pm_{ij}^\top \mathbf{M}_{ij,b})  \mathrm{d} \Pm_{ij} \Bigr)
\end{align}
where we used Jensen's inequality in \eqref{eqn:jensen}. Here, we observe that $\int_{\mathcal{P}_n} 2\beta  \theta_{ij,b}  \mathrm{tr}(\Pm_{ij}^\top \mathbf{M}_{ij,b})\mathrm{d} \Pm_{ij}$ is similar to the normalization constant of a Mallows model \cite{clemenccon2010kantorovich}, and it only depends on $\beta$ and $\theta_{ij,b} $. Hence, we conclude that 
\begin{align}
    Z_{ij} & \geq \prod_{b=1}^{B_{ij}} \exp \Bigl( \int_{\mathcal{P}_n} 2\beta  \theta_{ij,b}  \mathrm{tr}(\Pm_{ij}^\top \mathbf{M}_{ij,b})  \mathrm{d} \Pm_{ij} \Bigr) \\
    &:= \prod_{b=1}^{B_{ij}}  f(\beta, \theta_{ij,b})
\end{align}
where $f$ is an increasing function of $\beta, \theta_{ij,b}$ since $\mathrm{tr}(\Pm_{ij}^\top \mathbf{M}_{ij,b}) \geq 0$. This completes the proof.
\end{proof}

\section{Derivation of the Retraction Euler Integrator}
We start by recalling the SDE
\begin{align}
\mathrm{d} \xe_t = (-\mathbf{G}^{-1} \nabla_{\xe}  U_\lambda(\xe_t) + \boldsymbol{\Gamma}_t) \mathrm{d}t + \sqrt{2/\beta \mathbf{G}^{-1}} \mathrm{d} \mathrm{B}_t. \label{eqn:sdesupp}
\end{align}
By \cite{xifara2014langevin}, we know that 
\begin{align}
   \boldsymbol{\Gamma}_t = \frac1{2} \mathbf{G}^{-1} \nabla_{\xe} \log |\mathbf{G}| .
\end{align}
By using this identity in Equation~\ref{eqn:sdesupp}, we obtain:
\begin{align}
    \mathrm{d} \xe_t = -\mathbf{G}^{-1} (\nabla_{\xe}U_\lambda(\xe_t)  + \frac1{2}\nabla_{\xe} \log |\mathbf{G}|) \mathrm{d}t + \sqrt{2/\beta \mathbf{G}^{-1}} \mathrm{d} \mathrm{B}_t.
\end{align}
By using a similar notation to \cite{Liu2016}, we rewrite the above equation as follows:
\begin{align}
     \mathrm{d} \xe_t = -\mathbf{G}^{-1} (\nabla_{\xe}U_\lambda(\xe_t)  + \frac1{2}\nabla_{\xe} \log |\mathbf{G}|) \mathrm{d}t +  \mathbf{G}^{-1} \mathbf{M}^\top  \mathcal{N}(0, 2 \beta \mathbf{I})
\end{align}
where $\mathcal{N}$ denotes the Gaussian distribution and $[\mathbf{M}]_{ij} = \partial [\X]_{ij} / \partial[\xe]_{ij}$. We multiply each side of the above equation by $\mathbf{M}$ and use the property $\nabla_{\xe} = \mathbf{M}^\top \nabla_{\X}$ \cite{Liu2016}, which yields:
\begin{align}
     \mathrm{d} \X_t &= -\mathbf{M}\mathbf{G}^{-1}\mathbf{M}^\top \nabla_{\X}U(\X_t) \mathrm{d}t +  \mathbf{M} \mathbf{G}^{-1} \mathbf{M}^\top  \mathcal{N}(0, 2 \beta \mathbf{I}) \\
     &= \mathbf{M}(\mathbf{M}^\top\mathbf{M})^{-1}\mathbf{M}^\top \Bigl(- \nabla_{\X}U(\X_t) \mathrm{d}t + \mathcal{N}(0, 2 \beta \mathbf{I}) \Bigr).
\end{align}
Here we used the area formula (Equation 11 in the main paper) and the fact that $\mathbf{G} = \mathbf{M}^\top \mathbf{M}$.

\vspace{5pt}

The term $\mathbf{M}(\mathbf{M}^\top\mathbf{M})^{-1}\mathbf{M}^\top$ turns out the be the projection operator to the tangent space of $\X$ \cite{byrne2013}. Therefore, the usual geodesic integrator would consist of the following steps at iteration $k$:
\begin{itemize}
    \item Set a small step size $h$
    \item Compute the term $- h\nabla_{\X}U(\X_t)  + \sqrt{2h/\beta} \mathbf{Z}$, with $\mathbf{Z} \sim \mathcal{N}(0, \mathbf{I})$
    \item Obtain the `direction' by projecting the result of the previous step on the tangent space 
    \item Move the current iterate on the geodesic determined by the direction that was obtained in the previous step.
\end{itemize}

Unfortunately, the last step of the integrator above cannot be computed in the case of $\DP_n$. Therefore, we replace it with moving the current iterate by using the retraction operator, which yields the update equations given in the main paper.
\section{Details on the Synthetic Dataset}
We now give further details on the synthetic data used in the paper. We synthesize our data by displacing a $4\times 4$ 2D grid to simulate 5 different images and scrambling the order of correspondence. Shown in Fig.~\ref{fig:synth}{a}, the ground truth corresponds to an identity ordering where the $i^{th}$ element of each grid is mapped to $i^{th}$ node of the other. Note that the graph is fully connected but we show only consecutive connections. To simulate noisy data, we introduce $16$ random swaps into the ground truth as shown in Fig.~\ref{fig:synth}{b}. We also randomize the initialization in a similar way. On this data we first run a baseline method where instead of restricting ourselves to the Birkhoff Polytope, we use the paths of the orthogonal group $O(N)$. Our second baseline is the prosperous method of Maset~\etal~\cite{maset2017}. Note that, regardless of the initialization (Fig~\ref{fig:synth}{e},{f}) our method is capable of arriving at visually more satisfactory local minimum. This validates that for complex problems such as the one at hand, respecting the geometry of the constraints is crucial. Our algorithm is successful at that and hence is able to find high quality solutions. In the figure the more \textit{parallel} the lines are, the better, depicting closeness to the ground truth.
\insertimageStar{1}{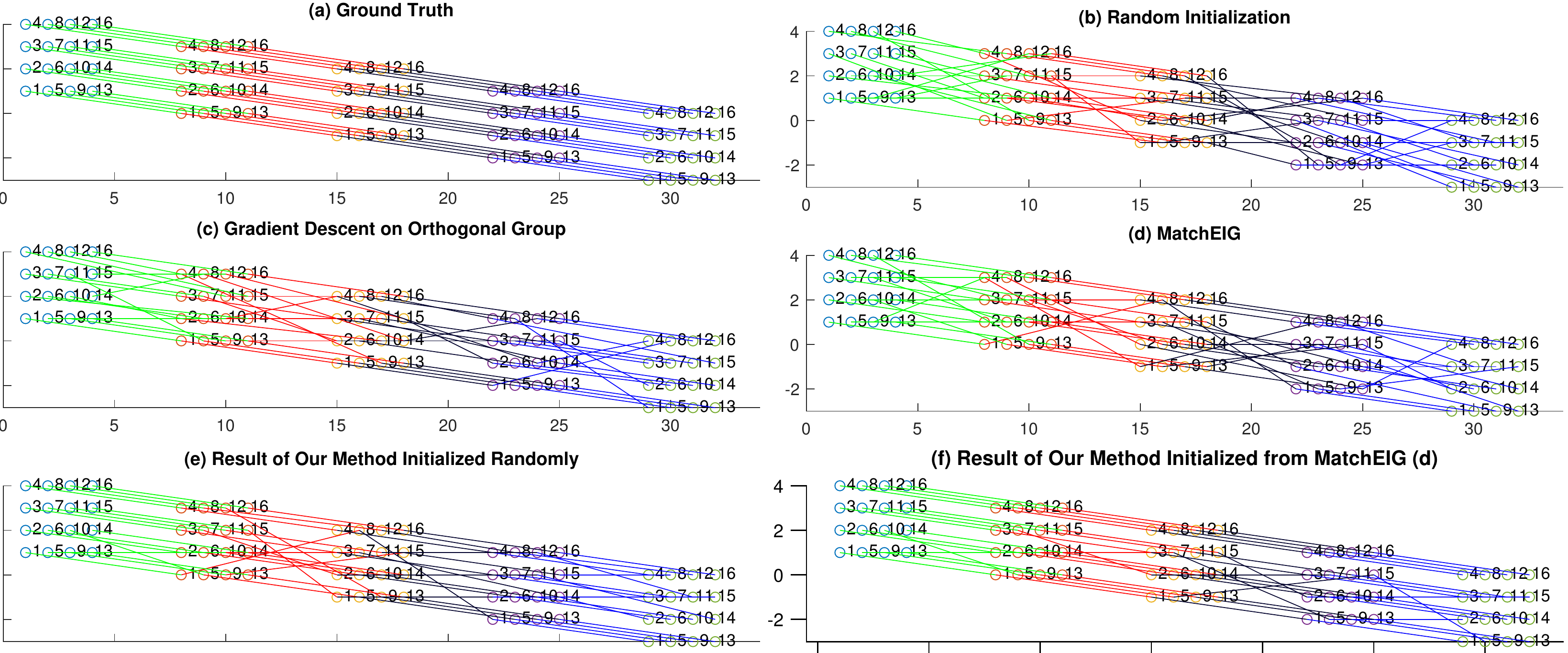}{Results from running our method on the synthetic dataset as well as some baseline methods such as Maset~\etal~\cite{maset2017} and minimizing our energy on the orthogonal group. Note that while using the same energy function, considering the Birkhoff Polytope (ours) rather than the orthogonal group leads to much more visually appealing results with higher recall.\vspace{-5mm}}{fig:synth}{hbtp}

\section{Examples on the Willow  Dataset}
\paragraph{Finding the Optimum Solution}
We now show matches visualized after running our synchronization on the Willow dataset~\cite{cho2013learning}. For the sake of space, we have omitted some of those results from the paper. Fig~\ref{fig:realsup} plots our findings, where our algorithm is able to converge from challenging initial matches. Note that these results only show our the MAP estimates explained in the main paper. It is possible to extend our approach with some form of geometric constraints as in Wang~\etal~\cite{wang2018multi}. We leave this as a future work.

\paragraph{Uncertainty Estimation}
There are not many well accepted standards in evaluating the uncertainty. Hence in the main paper, we resorted what we refer as the \textit{top-K} error. We will now briefly supply more details on this. The purpose of the experiment is to measure the quality of the sample proposals generated by our Birkhoff-RLMC. To this end, we introduce a random sampler that generates, $K$ arbitrary solutions that are highly likely to be irrelevant (bad proposals). These solutions alone do not solve the problem. However if we were to consider the result correct whenever either the found optimum or the random proposal contains the correct match i.e. append the additional $K-1$ samples to the solution set, then, even with a random sampler we are guaranteed to increase the recall. Similarly, samples from Birkhoff-RLMC will also yield higher recall. The question then is: Can Birkhoff-RLMC do better than a random sampler? To give an answer, we basically record the relative improvement in recall both for the random sampler and Birkhoff-RLMC. The \textit{top-K error} for different choices of $K$ is what we presented in Table 2 of the main paper. We further visualize these solutions in the columns (c) to (f) of the same table. To do so, we simply retain the assignments with the $K$-highest scores in the solutions $\X$. Note that the entire $\X$ acts as a confidence map itself (Column d). We then use the MATLAB command \textit{imagesc} on the ${\X}$ with the \textit{jet} colormap.

Next, we provide insights into how the sampler works in practice.
Fig.~\ref{fig:iterations} plots the objective value attained as the iterations of the Birkhoff-RLMC sampler proceeds. For different values of $\beta$ these plots look different. Here we use $\beta=0.08$ and show both on Motorbike and Winebottle samples (used above) and for 1000 iterations, the behaviour of sample selection. In the paper, we have accumulated these samples and estimated the confidence maps. Note that occasionally, the sampler can discover better solutions than the one being provided. This is due to two reasons: 1) rarely, we could jump over local minima, 2) the initial solution is a discrete one and it is often plausible to have a doubly stochastic (relaxed) solution matrices that have lower cost.


\insertimageStar{1}{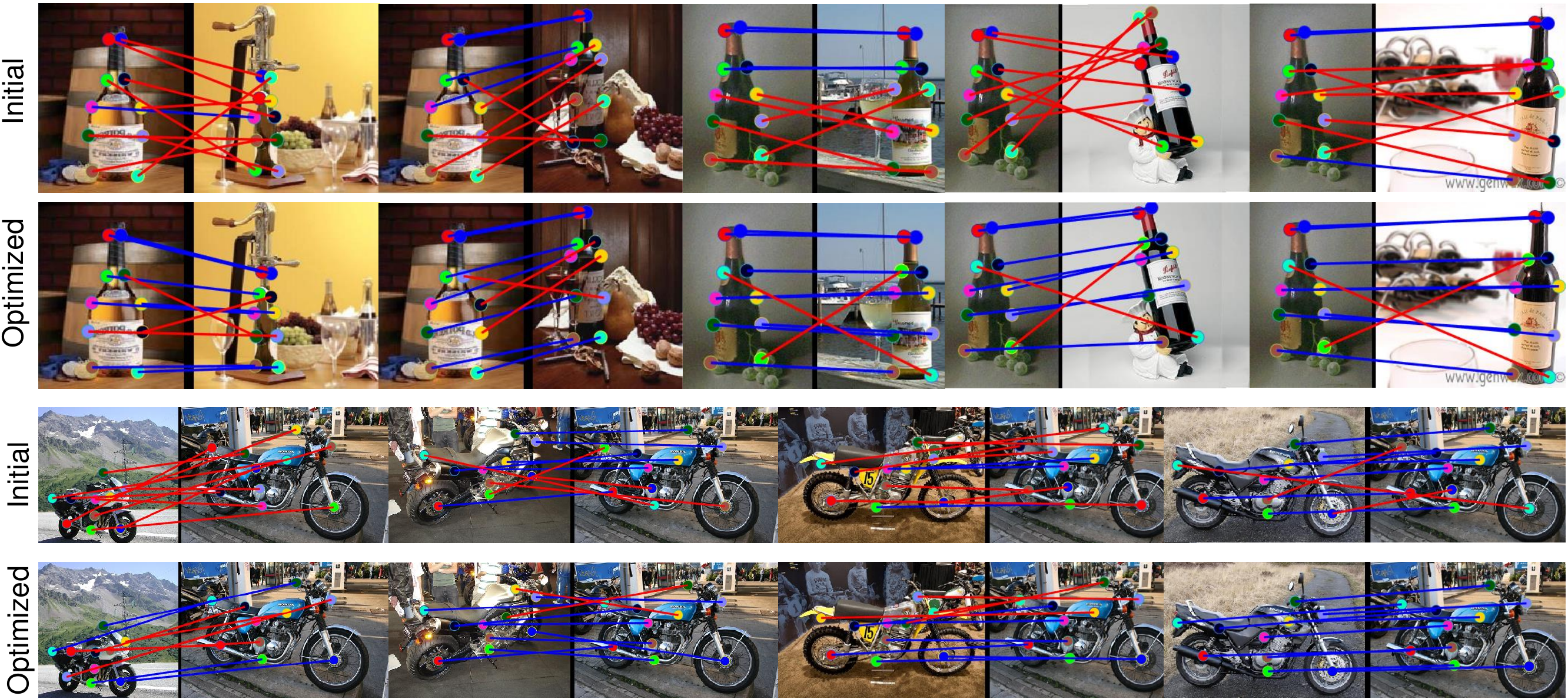}{Matching results on sample images from the \textit{Winebottle} (\textbf{top-2-rows}) and \textit{Motorbike} (\textbf{bottom-2-rows}) datasets in the form of correspondences. The first sub-row of each row shows the initial matches computed via running a pairwise Hungarian algorithm, whereas the second row (tagged \textit{Optimized}) shows our solution. Colored dots are the corresponding points. A line segment is shown in red if it is found to be wrong when checked against the ground truth. Likewise, a blue indicates a correct match. For both of the datasets we use the joint information by optimizing for the cycle consistency, whereas the pairwise solutions make no use of such multiview correspondences. Note that thanks to  the cycle consistency, once a match is correctly identified, its associated point has the tendency to be correctly matched across the entire dataset.\vspace{-5mm}}{fig:realsup}{hbtp}
\begin{figure}[hbtp]
\begin{center}
\subfigure[Sampling on the Winebottle]{
\includegraphics[width=0.39\columnwidth]{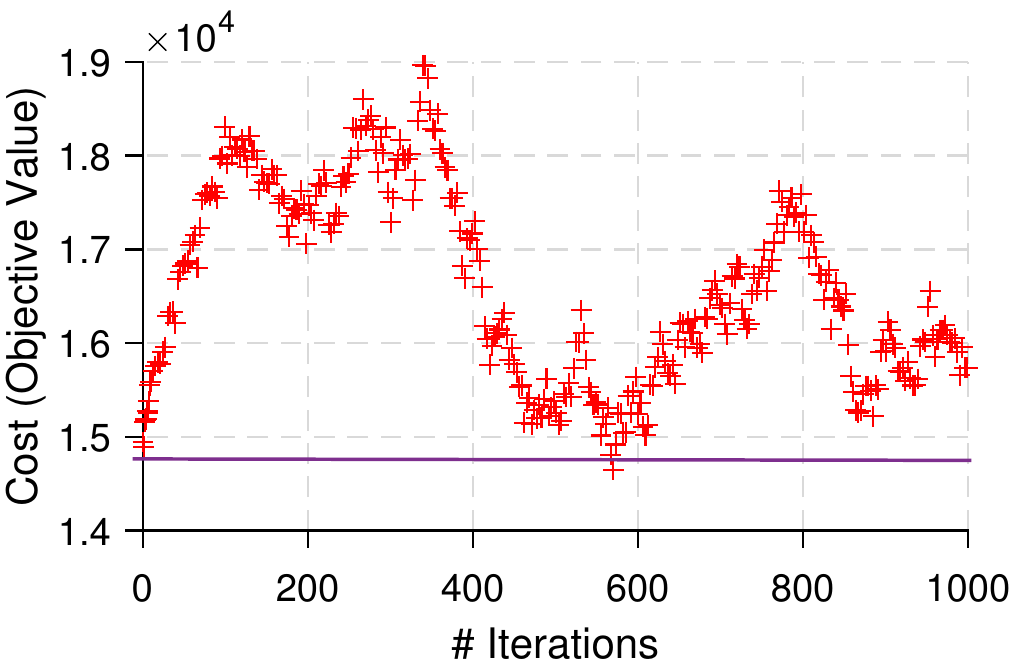}}\hspace{40pt}
\subfigure[Sampling on the Motorbike]{
\includegraphics[width=0.39\columnwidth]{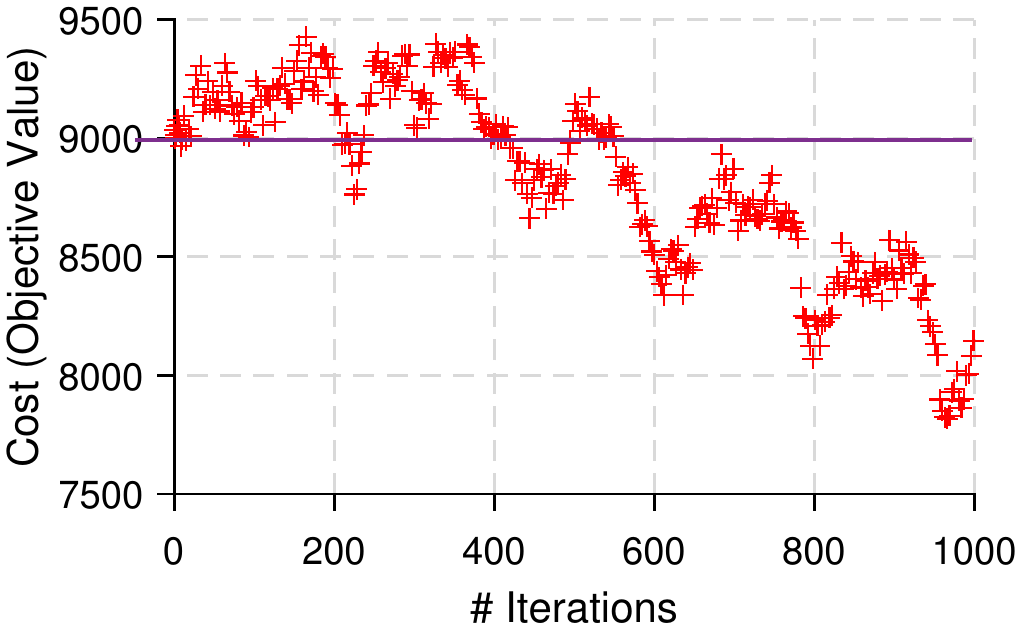}}
\end{center}
\caption{Iterates of our sampling algorithm. With $\beta=0.085$, our Birkhoff-RLMC sampler wanders around the local mode and draws samples on the Birkhoff Polytope. It can also happen that better solutions are found and returned. Purple line indicates the starting point, that is only a slight perturbation of the found solution.\vspace{-4mm}}
\label{fig:iterations}
\end{figure}

\section{Application to Establishing 3D Correspondences}
\label{sec:corresp3d}
We now apply our algorithm to solve correspondence problems on isolated 3D shapes provided by the synthetic Tosca dataset~\cite{bronstein2008numerical} that is composed of non-rigidly deforming 3D meshes of various objects such as cat, human, Centaur, dog, wolf and horse. The \textit{cat} object from this dataset along with ten of its selected ground truth correspondences is shown in Fig.~\ref{fig:tosca}. This is of great concern among the 3D vision community and the availability of the complete shapes plays well with the assumptions of our algorithm, i.e. permutations are \textit{total}. It is important to mention that when initial correspondences are good ($\sim 80\%$) all methods, including ours can obtain $100\%$ accuracy~\cite{huang2013consistent}. Therefore, to be able to put our method to test, we will intentionally degrade the initial correspondence estimation algorithm we use.


\vspace{3mm}\noindent\textbf{Initialization } Analogous to the 2D scenario, we obtain the initial correspondences by running a siamese Point-Net~\cite{qi2017pointnet} like network regressing the assignments between two shapes. Unlike 2D, we do not need to take care of occlusions, visibility or missing keypoint locations as 3D shapes can be full and clean as in this dataset.
On the average, we split half of the datasets for training and the other half for testing. Note that such amount of data is really insufficient to train this network, resulting in suboptimal predictions. We gradually downsample the point sets uniformly to sizes of $N_s=\{200, 50, 20\}$ to get the keypoints. At this point, it is possible to use sophisticated keypoint prediction strategies to establish the optimum cardinality and the set of points under correspondence. Note that it is sufficient to compute the keypoints per shape as we do not assume the presence of a particular order. Each keypoint is matched to another by the network prediction. The final stage of the prediction results in a soft assignment matrix on which we apply Hungarian algorithm to give rise to initial matches. 
\insertimageStar{1}{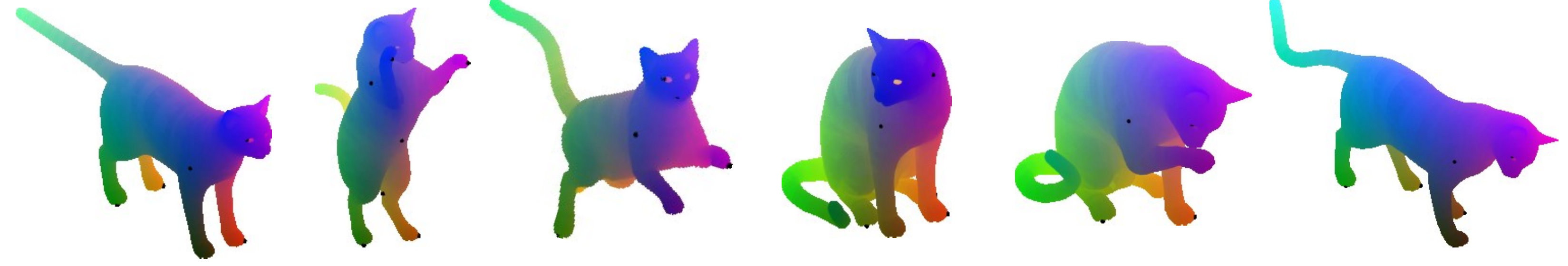}{Visualizations of the \textit{cat} object from the Tosca dataset with the ground truth correspondences depicted.}{fig:tosca}{hbtp}

\noindent\textbf{Preliminary Results } We run, on these initial sets of correspondences, our algorithm and the state of the art methods as described in the main paper. We count and report the percentage of correctly matched correspondences (recall) in Tab.~\ref{tab:tosca} for only two objects, \textit{Cat} and \textit{Michael}. Running on the full dataset is a future work. It is seen that our algorithm consistently outperforms the competing approaches. The difference is less visible when the number of samples start dominating the number of edges. This is because none of the algorithms are able to find enough consistency votes to correct for the errors. 



\begin{table}[htbp]
  \centering
  \caption{Correspondence refinement on 3D inputs. We show our results on the meshes of Tosca~\cite{bronstein2008numerical} dataset. The cells show the percentage of correctly detected matches (\textit{recall)}.\vspace{1mm}}
  \resizebox{\columnwidth}{!} {
    \begin{tabular}{lcccc|cccc|cccc}
    \toprule\toprule
          & \multicolumn{4}{c|}{N=200}    & \multicolumn{4}{c|}{N=50}     & \multicolumn{4}{c}{N=20} \\
\cmidrule{2-13}          & Initial & MatchEIG & Wang et al. & Ours  & Initial & MatchEIG & Wang et al. & Ours  & Initial & MatchEIG & Wang et al. & Ours \\
\cmidrule{2-13}    Cat   & 38.42\% & 32.92\% & 37.83\% & \textbf{39.08\%} & 53.38\% & 55.42\% & 56.13\% & \textbf{56.93\%} & 35.56\% & 37.33\% & 38.33\% & \textbf{40.33\%} \\
    Michael & 30.64\% & 34.39\% & 35.92\% & \textbf{36.12\%} & 46.40\% & 52.23\% & \textbf{54.93\%} & 54.62\% & 45.76\% & 51.05\% & 51.63\% & \textbf{55.95\%} \\
    \end{tabular}%
    }
  \label{tab:tosca}%
\end{table}%

\end{document}